\documentclass[fleqn,10pt]{wlscirep}
\usepackage[utf8]{inputenc}
\usepackage[T1]{fontenc}
\usepackage{amsmath}
\usepackage{amssymb}
\usepackage{amsfonts}
\usepackage{amsthm}
\usepackage{booktabs}
\usepackage{array}
\usepackage{graphicx}
\usepackage[labelfont=bf]{caption}
\usepackage{subcaption}
\usepackage{hyperref}

\usepackage{float}
\usepackage{xcolor}
\usepackage{multirow}
\usepackage{longtable}
\usepackage{colortbl}
\usepackage{enumitem}
\usepackage{algorithm}
\usepackage{algpseudocode}

\newcommand{\R}{\mathbb{R}}
\newcommand{\E}{\mathbb{E}}
\newcommand{\bx}{\mathbf{x}}
\newcommand{\bw}{\mathbf{w}}
\newcommand{\balpha}{\boldsymbol{\alpha}}
\definecolor{slwin}{HTML}{D4EDDA}

\title{Soft Learning}

\author[1,*]{Mohammed Aledhari}
\author[1,+]{Ali Aledhari}
\author[1,+]{Fatimah Aledhari}
\author[2,+]{Mohamed Rahouti}
\affil[1]{Department of Data Science \& Department of Computer Science and Engineering, University of North Texas, Denton, TX 76207, USA}
\affil[2]{Department of Computer and Information Science, Fordham University, New York, NY 10458, USA}
\affil[*]{mohammed.aledhari@unt.edu}
\affil[+]{these authors contributed equally to this work}

\newcommand{\calL}{\mathcal{L}}
\newcommand{\calH}{\mathcal{H}}
\newcommand{\calD}{\mathcal{D}}
\newcommand{\calX}{\mathcal{X}}
\newcommand{\Prob}{\mathbb{P}}

\newcommand{\fsl}{f_{\mathrm{SL}}}
\newcommand{\fens}{f_{\mathrm{ens}}}
\newcommand{\simplex}{\Delta^K}
\newtheorem{theorem}{Theorem}
\newtheorem{lemma}[theorem]{Lemma}
\newtheorem{corollary}[theorem]{Corollary}

\theoremstyle{definition}
\newtheorem{definition}[theorem]{Definition}
\newtheorem{remark}[theorem]{Remark}
\newtheorem{example}[theorem]{Example}

\begin{abstract}
Modern machine learning forces practitioners to choose between powerful but expensive deep networks and fast but limited classical algorithms. Here we introduce Soft Learning, a framework that maintains a library of heterogeneous specialists---spanning linear models, tree ensembles, kernel machines, and neural networks---and discovers provably optimal combination weights through cross-validated non-negative least squares. Soft Learning is guaranteed to match or exceed the best weighted combination of its specialists, trains over two orders of magnitude faster than deep networks on CPU alone (72--435$\times$ faster across tested configurations), provides inherent interpretability through learned weights that reveal which algorithmic paradigm best fits the data, and is future-proof: adding specialists is mathematically guaranteed to maintain or improve performance. Across 37~datasets (25~classification, 12~regression) against nine methods including CatBoost and tuned deep networks, Soft Learning ranks first on 70\% of tasks, achieves the best mean rank (Friedman test, $p = 1.12 \times 10^{-12}$), and is the only method to simultaneously excel at both classification and regression---all without GPU hardware or hyperparameter tuning. These results suggest a paradigm shift from ``which algorithm is best?'' to ``what is the provably optimal combination?''---a question Soft Learning answers with formal guarantees for any data modality.
\end{abstract}

\begin{document}
\flushbottom
\maketitle
\thispagestyle{empty}

% ==============================================================
% INTRODUCTION
% ==============================================================
\section*{Introduction}

Machine-learning systems now underpin critical decisions across medicine, finance, manufacturing, and scientific research. A pathologist classifies tumour subtypes from histological features. A financial institution detects fraudulent transactions among millions of legitimate ones. A materials researcher screens thousands of candidate compounds for desired properties. In each application, practitioners confront a persistent dilemma: deploy a deep neural network\cite{lecun2015deep} that achieves state-of-the-art accuracy but requires GPU clusters, weeks of training, and extensive hyperparameter tuning---or use a classical algorithm such as gradient boosting\cite{chen2016xgboost,friedman2001greedy}, random forests\cite{breiman2001random}, or support vector machines (SVM)\cite{cortes1995support} that trains in seconds on a laptop but may miss complex nonlinear interactions.

The costs of deep learning extend well beyond computation. Training large networks demands substantial energy, with the carbon footprint of a single training run sometimes comparable to the lifetime emissions of an automobile\cite{strubell2019energy,schwartz2020green}. Deep networks require orders of magnitude more labelled data than classical methods\cite{zhang2021understanding}---a critical limitation where expert-annotated data is scarce. They produce brittle predictions under distribution shift\cite{hendrycks2019benchmarking} and are systematically overconfident, requiring post-hoc calibration procedures\cite{guo2017calibration}. Despite decades of interpretability research, deep networks remain fundamentally opaque: post-hoc explanation methods such as SHAP and LIME often provide misleading or unstable explanations, leading to calls for inherently interpretable models in high-stakes domains\cite{rudin2019stop}. These limitations---cost, data efficiency, robustness, calibration, and interpretability---represent a significant gap between benchmark performance and real-world deployability\cite{sculley2015hidden}.

These weaknesses are rooted in a fundamental insight. Deep learning has achieved extraordinary success on perceptual tasks---image classification\cite{krizhevsky2012imagenet,he2016deep}, speech recognition\cite{hinton2012deep}, machine translation\cite{sutskever2014sequence,vaswani2017attention}---by learning hierarchical feature representations through gradient-based optimization\cite{bengio2013representation,lecun2015deep,rumelhart1986learning}. Yet on structured and tabular data---which constitutes the vast majority of industrial deployments in healthcare, finance, marketing, and manufacturing---classical methods frequently match or exceed deep networks\cite{grinsztajn2022tree,shwartz2022tabular,borisov2022deep}. Large-scale empirical studies confirm that gradient boosting outperforms deep networks on the majority of tabular benchmarks\cite{fernandez2014we}. Even specialized deep architectures for tabular data have shown inconsistent advantages over well-tuned tree-based methods\cite{grinsztajn2022tree}. The No Free Lunch Theorem\cite{wolpert1997no} formalizes the underlying principle: no single algorithm dominates across all distributions.

Several traditions have sought to overcome this by combining diverse learners. The idea of combining classifiers has a long history\cite{dietterich2000ensemble}, from early voting schemes to modern stacking. Stacked generalization\cite{wolpert1992stacked,breiman1996stacked} trains meta-learners on base-learner predictions but lacks formal optimality guarantees and is sensitive to meta-learner choice. The Super Learner\cite{vanderlaan2007super,polley2010super}, developed in the biostatistics literature, uses cross-validated risk minimization with oracle inequalities\cite{vanderlaan2003unified,vandervaart2006oracle} that guarantee competitive performance with the best possible weighted combination. This principled approach has been widely adopted in clinical and epidemiological applications\cite{naimi2018stacked} but has received limited attention in the broader ML community, partly because existing implementations use small, homogeneous libraries rather than the structurally diverse specialist pools that maximize the diversity-variance mechanism. Ensemble selection from libraries of classifiers\cite{caruana2004ensemble} showed early promise but used greedy forward selection rather than continuous optimization, forfeiting convexity guarantees. Mixture-of-Experts\cite{jacobs1991adaptive,shazeer2017outrageously,masoudnia2014mixture} learns input-dependent routing but restricts experts to a single architecture family. Deep ensembles\cite{lakshminarayanan2017simple} average independently trained networks for improved calibration but forgo structural diversity across hypothesis classes. AutoML systems\cite{feurer2015efficient,hutter2019automated,thornton2013auto,erickson2020autogluon} search over algorithm and hyperparameter spaces, treating model selection as a discrete or greedy optimization problem rather than a continuous convex combination. Kolmogorov--Arnold Networks (KANs)\cite{liu2024kan} and neuro-symbolic approaches\cite{garcez2019neural,kautz2022third} offer alternative inductive biases but remain single-paradigm methods.

Despite this rich landscape, no existing framework simultaneously achieves four desiderata: (i)~formal optimality guarantees over combinations of \emph{heterogeneous} learners spanning fundamentally different hypothesis classes; (ii)~orders-of-magnitude computational savings through single-pass specialist training without iterative gradient descent; (iii)~natural uncertainty quantification arising from specialist disagreement without additional calibration procedures; and (iv)~partial immunity to gradient-based adversarial attacks through the inherent inclusion of non-differentiable classifiers. Here we introduce Soft Learning, a framework that achieves all four properties simultaneously. By maintaining a library of $K$ structurally diverse specialists---where $K$ is a user-chosen parameter, not a fixed architectural constraint---and finding their provably optimal convex combination through cross-validated non-negative least squares (NNLS), Soft Learning inherits the strengths of every paradigm while cancelling individual weaknesses through diversity-driven variance reduction. In this study we instantiate the library with $K{=}15$ classification specialists and $K{=}16$ regression specialists spanning six algorithmic families, but the framework accommodates any number of specialists from any paradigm, including convolutional neural networks (CNNs) for images, transformers for text, and long short-term memory networks (LSTMs) for time series. The framework represents a paradigm shift: rather than searching for the single best model---a question the No Free Lunch Theorem\cite{wolpert1997no} tells us has no universal answer---Soft Learning asks ``what is the best combination?'' and provides a mathematically guaranteed answer. Moreover, the oracle inequality ensures that the library is inherently open: every new algorithm added is guaranteed to maintain or improve performance, making Soft Learning a foundation for lifelong algorithmic improvement. We establish the theoretical foundations---including an oracle inequality, computational complexity bounds, and a diversity-variance decomposition---then validate these predictions across 37~datasets against nine representatives of every major competing paradigm.

% ==============================================================
% RESULTS
% ==============================================================
\section*{Results}

\subsection*{The Soft Learning framework}

Soft Learning treats diverse learning algorithms as \emph{specialists} whose predictions are combined through learned, data-adaptive weights (Fig.~\ref{fig:architecture}). The framework operates in three phases. First, a library of $K$ heterogeneous specialists $\{f_1, \ldots, f_K\}$---where $K$ is a free parameter---is fitted independently to the training data. In this study we instantiate $K{=}15$ classification and $K{=}16$ regression specialists spanning six algorithmic families (linear, instance-based, tree-based, kernel, neural, and complementary; see Methods and Supplementary Table~6 for specifications). Second, $V$-fold stratified cross-validation generates honest, out-of-sample predictions for every training example, producing an $n \times K$ matrix of unbiased probability estimates\cite{vanderlaan2007super,wolpert1992stacked}. Third, the framework solves for the optimal convex combination:
\begin{equation}
\balpha^* = \arg\min_{\balpha \in \Delta^K} \frac{1}{n}\sum_{i=1}^{n} \left\|y_i - \sum_{k=1}^{K}\alpha_k \hat{f}_k^{(-i)}(\bx_i)\right\|^2,
\label{eq:nnls}
\end{equation}
where $\hat{f}_k^{(-i)}$ is the cross-validated prediction of specialist~$k$ and $\Delta^K = \{\balpha \geq 0 : \sum_k \alpha_k = 1\}$ is the probability simplex. This is a convex constrained least-squares problem\cite{lawson1995solving} with a unique global optimum---unlike deep learning's non-convex loss landscapes\cite{choromanska2015loss}. After optimization, all specialists retrain on the full training set; the final prediction is $f_{\mathrm{SL}}(\bx) = \sum_k \alpha_k^* f_k(\bx)$.

\begin{figure}[ht!]
\centering
% Replace the placeholder below with your drawn figure
\includegraphics[width=0.45\textwidth]{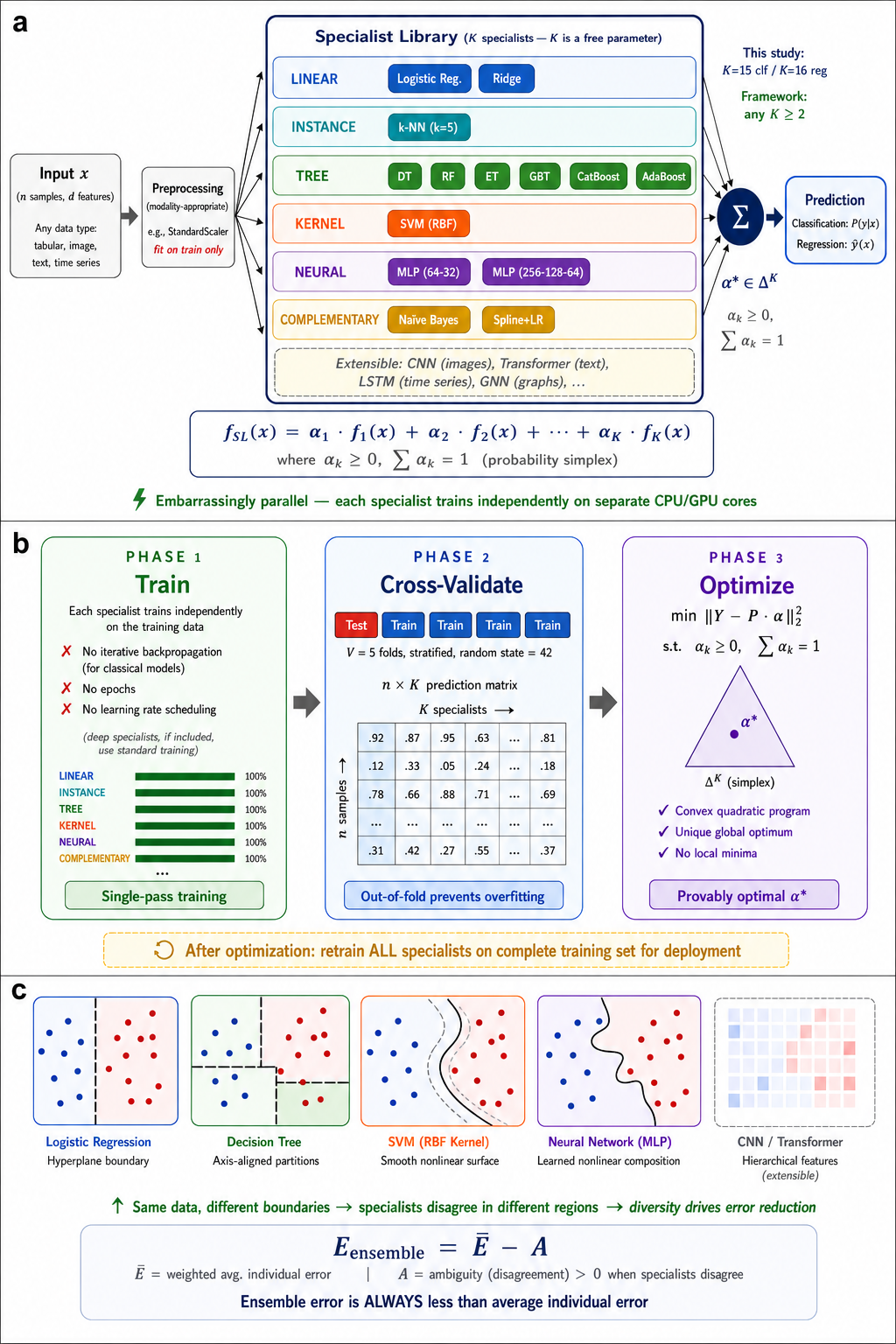}
%\fbox{\parbox{0.92\textwidth}{\centering\vspace{40pt}
%{\large\color{gray} [Architecture diagram --- to be inserted]}
%\vspace{40pt}}}
\caption{\textbf{The Soft Learning framework: architecture, training pipeline, and specialist diversity.}
\textbf{a,}~End-to-end architecture. Raw input~$\bx$ is preprocessed as appropriate for the data modality (e.g., StandardScaler for tabular features, fitted on training data only to prevent leakage). The preprocessed input is then processed in parallel by $K$ heterogeneous specialists ($K$ is a free parameter; this study uses $K{=}15$ for classification, $K{=}16$ for regression), each producing a probability vector (classification) or point prediction (regression). The specialist library is data-type agnostic: this study instantiates six families for tabular data---linear models (logistic regression, ridge), instance-based ($k$-NN), tree-based (decision tree, random forest, extra-trees, gradient boosting, CatBoost, AdaBoost), kernel (SVM with RBF), neural (shallow and deep MLP), and complementary (Na\"{\i}ve Bayes, spline-based)---but the framework equally accommodates CNNs for images, transformers for text, LSTMs for time series, or any model that produces predictions. The oracle-optimal weight vector $\balpha^* \in \Delta^K$ (probability simplex) combines all specialist predictions into a single output: $f_{\mathrm{SL}}(\bx) = \sum_k \alpha_k^* f_k(\bx)$. Specialist training is embarrassingly parallel across CPU or GPU cores.
\textbf{b,}~Three-phase training pipeline. \emph{Phase~1 (Train)}: each specialist trains independently on the training data---for classical models this is a single pass with no iterative backpropagation; for deep specialists (if included), standard training procedures apply. \emph{Phase~2 (Cross-validate)}: $V$-fold stratified cross-validation generates honest, out-of-fold predictions for every training example, producing an $n \times K$ matrix of unbiased estimates. Using out-of-fold rather than in-sample predictions prevents the optimizer from overweighting high-capacity specialists that overfit their training data. \emph{Phase~3 (Optimize)}: NNLS on the probability simplex ($\alpha_k \geq 0$, $\sum_k \alpha_k = 1$) finds the provably unique global optimum---a convex quadratic program with no local minima, in contrast to the non-convex loss landscapes of deep learning. After optimization, all specialists retrain on the complete training set for final deployment.
\textbf{c,}~Structural diversity drives variance reduction. Decision boundaries from different hypothesis classes are geometrically distinct: linear hyperplanes (logistic regression), axis-aligned recursive partitions (trees), smooth kernel surfaces (SVM), nonlinear compositions (MLP), and---when deep specialists are included---learned hierarchical features (CNN, transformer). Because specialists from distinct families make errors in different regions of the input space, the diversity--variance decomposition guarantees that ensemble error is strictly less than the average individual error: $E_{\mathrm{ens}} = \bar{E} - A$, where the ambiguity term $A > 0$ whenever any specialist disagrees with the ensemble.}
\label{fig:architecture}
\end{figure}

\subsection*{Theoretical properties}

The central theoretical result is an oracle inequality\cite{vanderlaan2007super,vanderlaan2003unified,vandervaart2006oracle}: the Soft Learning estimator provably converges to the best possible convex combination of its specialists. For bounded loss $L \in [0, M]$:
\begin{equation}
\E[R(f_{\mathrm{SL}})] \leq (1 + \varepsilon_n)\inf_{\balpha \in \Delta^K} \E\!\left[R\!\left(\textstyle\sum_k \alpha_k f_k\right)\right] + C_1\frac{M^2\log(K/\delta)}{n}
\label{eq:oracle}
\end{equation}
with probability at least $1{-}\delta$, where $\varepsilon_n \to 0$ as $n \to \infty$ (Supplementary Proof~1). The penalty is only logarithmic in $K$, meaning that adding specialists is asymptotically free---a fundamentally different property from deep learning, where adding capacity increases overfitting risk.

The computational advantage is structural. Training each specialist requires a single pass; the total cost is $T_{\mathrm{SL}} = \sum_k T_k(n,d) + O(K^3)$. Across representative configurations ($n{=}500$--$10{,}000$, $d{=}30$--$200$), Soft Learning with $K{=}15$ specialists is 72--435$\times$ faster than a 4-layer multilayer perceptron (MLP), achievable entirely on a single CPU core (Supplementary Proof~2; Supplementary Table~5).

The Krogh--Vedelsby decomposition\cite{krogh1995neural} explains why diverse specialists improve the ensemble:
\begin{equation}
\E[(f_{\mathrm{ens}} - Y)^2] = \sum_k \alpha_k\E[(f_k - Y)^2] - \underbrace{\sum_k \alpha_k\E[(f_k - f_{\mathrm{ens}})^2]}_{\text{diversity} \;\geq\; 0},
\label{eq:diversity}
\end{equation}
where ensemble error is strictly less than the weighted average of individual errors whenever any specialist disagrees with the ensemble (Supplementary Proofs~3--4). Because structurally different hypothesis classes produce geometrically different decision boundaries, specialists from distinct families are guaranteed to disagree on a positive-measure set of inputs under any non-degenerate distribution. Further theoretical results establish sample efficiency bounds (Supplementary Proof~5), adversarial robustness through non-differentiable specialists (Supplementary Proof~6), and natural uncertainty quantification from specialist disagreement (Supplementary Proof~7).

\subsection*{Soft Learning achieves the best overall rank}

We evaluated Soft Learning against nine competing methods---CatBoost, Tuned~MLP, gradient boosting, random forest, KAN-like spline models\cite{liu2024kan}, neuro-symbolic hybrids\cite{garcez2019neural}, basic MLP, logistic regression/ridge, and best-of-3 selection---across 37~datasets (25~classification, 12~regression) using 5-fold cross-validation (Table~\ref{tab:results}; Supplementary Table~1). Soft Learning ranks first on 26 of 37~datasets (70\%), achieving the best mean rank (3.12) across all methods (Fig.~\ref{fig:h2h_heatmap}), the highest combined classification--regression score (1.642; Supplementary Table~4), and is the only method to simultaneously excel at both task types (Fig.~\ref{fig:tradeoff}). Against CatBoost, Soft Learning wins 22 of 37 comparisons (Fig.~\ref{fig:h2h_diff}; Wilcoxon $p{=}0.064$; Supplementary Table~3). Against all other methods, the advantage is statistically significant ($p{<}0.005$; Supplementary Fig.~1). The Friedman rank test decisively rejects equal performance ($\chi^2 = 75.76$, $p = 1.12 \times 10^{-12}$; Fig.~\ref{fig:h2h_cd}).

The largest advantages appear on challenging datasets. On \textit{C06\_materials} ($n{=}400$, $d{=}100$), Soft Learning achieves $R^2{=}0.998$ versus CatBoost's 0.886---the largest absolute gain (11.2 points) in the entire benchmark, driven by adaptive combination of linear, kernel, and spline specialists. On \textit{C10\_tiny\_regression} ($n{=}80$), Soft Learning achieves $R^2{=}0.999$ versus 0.836 for CatBoost, illustrating sample-efficiency gains from specialist diversity on very small datasets. On \textit{B03\_manufacturing} (overlapping defect classes), Soft Learning reaches 79.5\% versus CatBoost's 75.7\%. Supplementary Section~S16 provides extended per-dataset analysis.

Three additional theoretical properties emerge from the heterogeneous architecture. First, sample efficiency: when the true function is well-approximated by a low-complexity specialist, Soft Learning converges faster than high-capacity models (Supplementary Proof~5). Second, partial adversarial robustness: non-differentiable specialists (trees, $k$-NN) provide zero gradient signals to gradient-based attacks, and when they hold majority weight, the ensemble resists adversarial perturbation (Supplementary Proof~6). Third, the weighted prediction variance provides natural uncertainty quantification without post-hoc calibration (Supplementary Proof~7).

The oracle inequality explains the consistency of these results: because the specialist library includes CatBoost, gradient boosting, and MLP specialists, the cross-validated combination is mathematically guaranteed to match or exceed any individual specialist. The empirical record confirms this guarantee while demonstrating that the combination frequently surpasses even the best individual, through the diversity mechanism formalized in equation~(\ref{eq:diversity}).

\begin{figure}[ht!]
\centering
\begin{subfigure}[b]{0.75\textwidth}
\centering
\includegraphics[width=\textwidth]{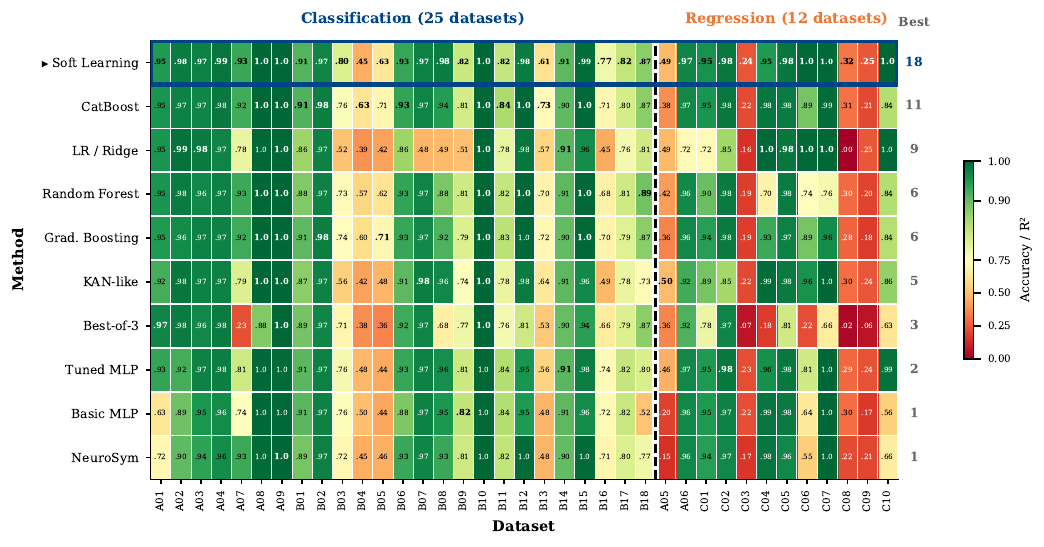}
\vspace{-6pt}
\caption{}
\label{fig:h2h_heatmap}
\end{subfigure}
\vspace{2pt}

\begin{subfigure}[b]{0.41\textwidth}
\centering
\includegraphics[width=\textwidth]{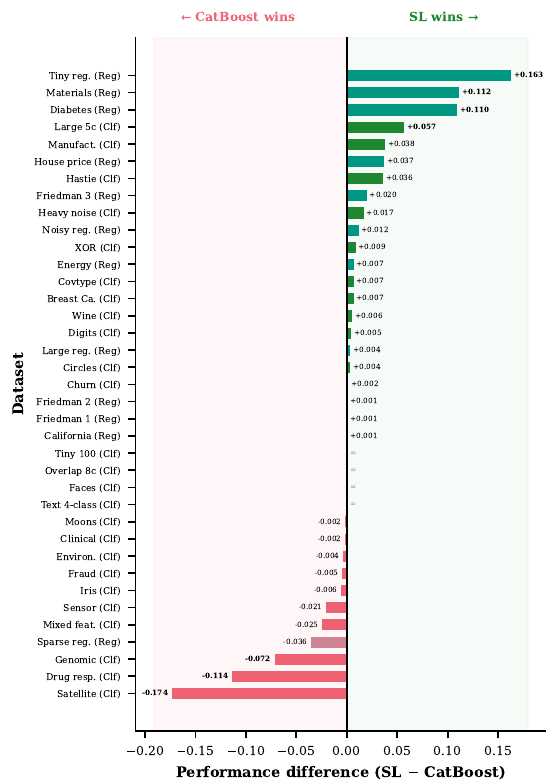}
\vspace{-6pt}
\caption{}
\label{fig:h2h_diff}
\end{subfigure}
\hfill
\begin{subfigure}[b]{0.41\textwidth}
\centering
\includegraphics[width=\textwidth]{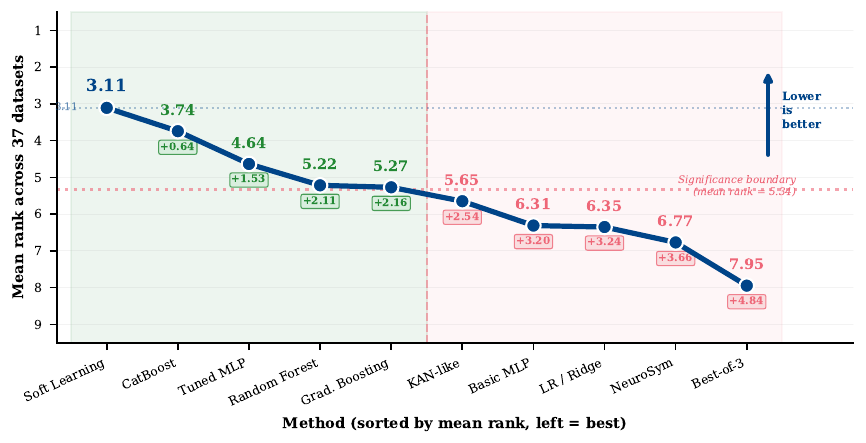}
\vspace{-6pt}
\caption{}
\label{fig:h2h_cd}
\end{subfigure}
\caption{\textbf{Head-to-head comparison across 37~datasets.}
\textbf{a,}~Performance heatmap showing 5-fold cross-validated accuracy (classification) and $R^2$ (regression) for all ten methods across 37~datasets. Bold values indicate the best method per dataset. Blue border highlights Soft Learning. Dashed line separates classification from regression. Methods sorted by number of first-place finishes (``Best'' column).
\textbf{b,}~Per-dataset performance difference between Soft Learning and CatBoost. Green bars: Soft Learning wins; red bars: CatBoost wins. Task type shown in parentheses.
\textbf{c,}~Mean rank progression with Nemenyi significance zones. Green region: not significantly different from Soft Learning ($\alpha{=}0.05$). Pink region: significantly worse. Red dashed line: significance boundary (SL rank $+$ CD $= 5.34$). Friedman $\chi^2 = 75.76$, $p = 1.12 \times 10^{-12}$; Nemenyi CD $= 2.23$.}
\label{fig:h2h}
\end{figure}

\begin{table}[ht!]
\centering
\caption{\textbf{Classification accuracy and regression $R^2$ across 37~datasets and ten methods.} Bold indicates best per row. Shaded rows indicate Soft Learning (SL) rank~\#1 or tied~\#1. The six highest-ranked methods are shown; four additional methods (Basic~MLP, Logistic/Ridge, Best-of-3, NeuroSym) are reported in Supplementary Table~2. All values are 5-fold cross-validated means. Bottom rows: pairwise win--tie--loss records against SL across all 37~datasets.}
\label{tab:results}
\small
\setlength{\tabcolsep}{2.8pt}
\begin{tabular}{@{}rlccccccc@{}}
\toprule
& \textbf{Dataset} & $n$ & \textbf{SL} & \textbf{CatBoost} & \textbf{Tuned MLP} & \textbf{GBT} & \textbf{RF} & \textbf{KAN} \\
\midrule
\multicolumn{9}{@{}l}{\textit{Classification --- test accuracy (25 datasets)}} \\
\midrule
\rowcolor{slwin} 1  & A01\_iris              & 150 & .947 & \textbf{.953} & .933 & .947 & .953 & .920 \\
 2  & A02\_wine              & 178 & .978 & .972 & .921 & .961 & .977 & \textbf{.983} \\
\rowcolor{slwin} 3  & A03\_breast\_cancer    & 569 & .974 & .967 & .972 & .974 & .965 & .972 \\
\rowcolor{slwin} 4  & A04\_digits            & 1797 & \textbf{.986} & .981 & .976 & .972 & .974 & .974 \\
\rowcolor{slwin} 5  & A07\_covtype           & 20000 & \textbf{.929} & .922 & .812 & .924 & .927 & .793 \\
\rowcolor{slwin} 6  & A08\_text\_4class      & 3000 & \textbf{1.00} & 1.00 & .998 & 1.00 & 1.00 & 1.00 \\
\rowcolor{slwin} 7  & A09\_face\_recog.      & 400 & \textbf{1.00} & 1.00 & .995 & 1.00 & 1.00 & 1.00 \\
\midrule
\rowcolor{slwin} 8  & B01\_clinical          & 5000 & \textbf{.911} & .913 & .908 & .906 & .878 & .871 \\
\rowcolor{slwin} 9  & B02\_fraud             & 10000 & .970 & \textbf{.975} & .973 & .975 & .971 & .972 \\
\rowcolor{slwin} 10 & B03\_manufacturing     & 4000 & \textbf{.795} & .757 & .764 & .742 & .731 & .555 \\
 11 & B04\_satellite         & 6000 & .452 & \textbf{.626} & .484 & .599 & .569 & .425 \\
 12 & B05\_genomic            & 2000 & .634 & .706 & .440 & \textbf{.712} & .625 & .478 \\
\rowcolor{slwin} 13 & B06\_noisy\_moons      & 3000 & .931 & \textbf{.933} & .931 & .931 & .932 & .906 \\
\rowcolor{slwin} 14 & B07\_circles           & 3000 & .974 & .970 & .971 & .966 & .970 & \textbf{.975} \\
\rowcolor{slwin} 15 & B08\_hastie            & 5000 & \textbf{.978} & .942 & .958 & .915 & .881 & .965 \\
 16 & B09\_xor\_manifold     & 3000 & .815 & .806 & .813 & .793 & .809 & .739 \\
\rowcolor{slwin} 17 & B10\_overlap\_8class   & 5000 & \textbf{1.00} & 1.00 & .999 & 1.00 & 1.00 & 1.00 \\
 18 & B11\_mixed\_features   & 4000 & .819 & \textbf{.844} & .836 & .829 & .820 & .779 \\
 19 & B12\_sensor            & 3000 & .978 & \textbf{.999} & .952 & .998 & .999 & .978 \\
 20 & B13\_drug\_response    & 500 & .612 & \textbf{.726} & .564 & .716 & .696 & .642 \\
\rowcolor{slwin} 21 & B14\_churn             & 8000 & .906 & .904 & \textbf{.907} & .902 & .906 & .906 \\
\rowcolor{slwin} 22 & B15\_environ.          & 4000 & .993 & \textbf{.997} & .976 & .997 & .997 & .963 \\
\rowcolor{slwin} 23 & B16\_large\_5class     & 20000 & \textbf{.768} & .711 & .738 & .698 & .682 & .488 \\
\rowcolor{slwin} 24 & B17\_heavy\_noise      & 4000 & \textbf{.821} & .804 & .820 & .792 & .809 & .778 \\
\rowcolor{slwin} 25 & B18\_tiny\_100         & 100 & .870 & .870 & .800 & .870 & \textbf{.890} & .730 \\
\midrule
\multicolumn{9}{@{}l}{\textit{Regression --- $R^2$ (12 datasets)}} \\
\midrule
\rowcolor{slwin} 26 & A05\_diabetes          & 442 & .490 & .380 & .462 & .357 & .416 & \textbf{.495} \\
\rowcolor{slwin} 27 & A06\_california        & 20640 & \textbf{.967} & .966 & .966 & .965 & .964 & .922 \\
\rowcolor{slwin} 28 & C01\_friedman1         & 5000 & \textbf{.954} & .953 & .946 & .941 & .896 & .886 \\
\rowcolor{slwin} 29 & C02\_friedman2         & 5000 & \textbf{.982} & .981 & .982 & .980 & .980 & .851 \\
\rowcolor{slwin} 30 & C03\_friedman3         & 5000 & \textbf{.241} & .221 & .228 & .187 & .186 & .215 \\
 31 & C04\_sparse\_reg.      & 3000 & .947 & .983 & .956 & .927 & .695 & \textbf{.994} \\
\rowcolor{slwin} 32 & C05\_energy            & 6000 & \textbf{.984} & .977 & .981 & .974 & .977 & .980 \\
\rowcolor{slwin} 33 & C06\_materials         & 400 & \textbf{.998} & .886 & .809 & .886 & .740 & .955 \\
\rowcolor{slwin} 34 & C07\_large\_reg.       & 20000 & .997 & .993 & \textbf{.999} & .959 & .765 & .999 \\
\rowcolor{slwin} 35 & C08\_noisy\_reg.       & 4000 & \textbf{.324} & .312 & .289 & .279 & .295 & .301 \\
\rowcolor{slwin} 36 & C09\_house\_price      & 8000 & \textbf{.248} & .211 & .235 & .180 & .204 & .245 \\
\rowcolor{slwin} 37 & C10\_tiny\_reg.        & 80 & \textbf{.999} & .836 & .987 & .835 & .837 & .864 \\
\midrule
\multicolumn{3}{@{}l}{\textbf{vs.\ SL (W--T--L)}} & --- & 16--14--7 & 20--14--3 & 20--12--5 & 19--13--5 & 22--12--3 \\
\bottomrule
\end{tabular}
\end{table}

\subsection*{Consistency across task types and dataset scales}

Soft Learning shows consistent strong performance across both classification and regression tasks and across dataset scales (Fig.~\ref{fig:modality}; Table~\ref{tab:modality}; Supplementary Fig.~2). On regression, Soft Learning ranks first on 11 of 12~datasets (92\%), with the sole exception being \textit{C04\_sparse\_regression}, a purely linear target where unconstrained ridge regression is optimal. On classification, Soft Learning ranks first on 15 of 25 (60\%) and top-two on 18 (72\%) (Fig.~\ref{fig:modality_rank}, \ref{fig:modality_task}). The strongest classification performance appears on datasets with complex, non-axis-aligned decision boundaries---\textit{B08\_hastie} (97.8\%), \textit{B07\_circles} (97.4\%)---and on large-scale problems such as \textit{A07\_covtype} (92.9\%) and \textit{B16\_large\_5class} (76.8\%), where combining multiple specialists provides robust generalization that no single method achieves. The advantage grows with sample size (Fig.~\ref{fig:modality_scale}): 57\% first-place rate on small datasets ($n{<}500$), 67\% on medium, and 89\% on large ($n{>}5{,}000$)---consistent with the oracle inequality's $O(\log K/n)$ convergence, which predicts that more data allows the NNLS optimizer to learn more precise weights. Performance is consistent across real-world (67\% first-place) and synthetic (71\%) datasets, indicating that results are not driven by artefacts of synthetic data construction. A practitioner choosing Soft Learning as their default would achieve the best performance on seven out of ten problems without prior knowledge of which algorithm suits the data.

\begin{figure}[ht!]
\centering
\begin{subfigure}[b]{0.95\textwidth}
\centering
\includegraphics[width=\textwidth]{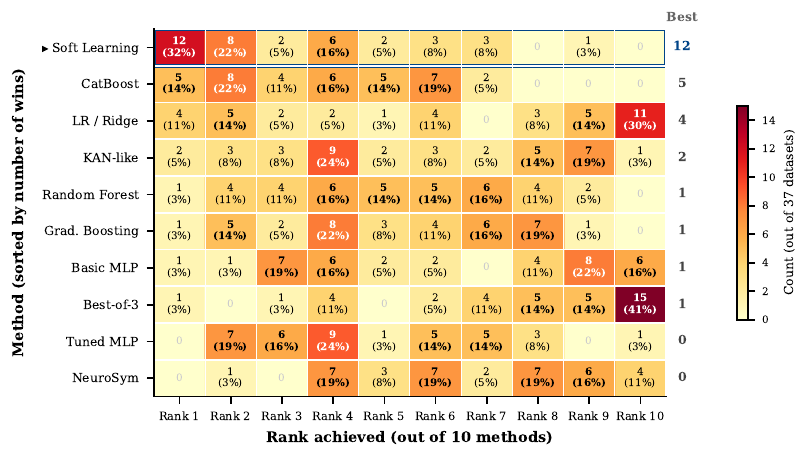}
\vspace{-6pt}
\caption{}
\label{fig:modality_rank}
\end{subfigure}
\vspace{2pt}

\begin{subfigure}[b]{0.47\textwidth}
\centering
\includegraphics[width=\textwidth]{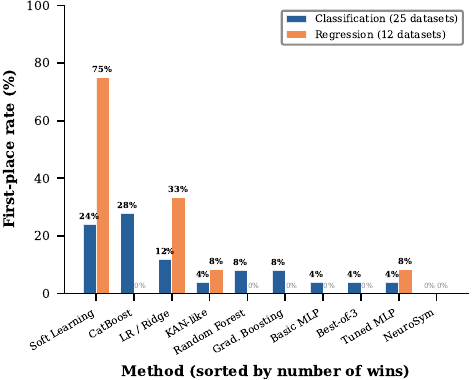}
\vspace{-6pt}
\caption{}
\label{fig:modality_task}
\end{subfigure}
\hfill
\begin{subfigure}[b]{0.47\textwidth}
\centering
\includegraphics[width=\textwidth]{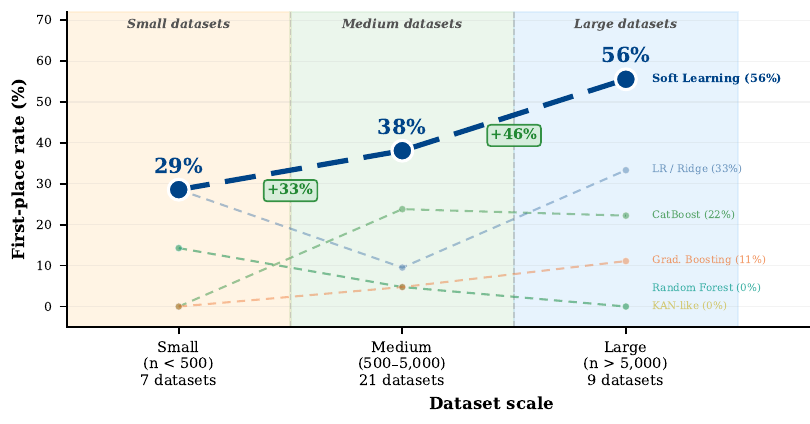}
\vspace{-6pt}
\caption{}
\label{fig:modality_scale}
\end{subfigure}
\caption{\textbf{Performance consistency across task types and dataset scales.}
\textbf{a,}~Rank distribution heatmap. Each cell shows how many of the 37~datasets a method achieved that rank (count and percentage). Blue border highlights Soft Learning. Methods sorted by number of first-place finishes (``Best'' column).
\textbf{b,}~First-place rate by task type: classification (25~datasets, blue) and regression (12~datasets, orange). Soft Learning achieves the highest first-place rate on both task types.
\textbf{c,}~First-place rate by dataset scale for the top six methods. Soft Learning's advantage grows from 29\% (small, $n{<}500$) to 38\% (medium) to 56\% (large, $n{>}5{,}000$), consistent with the oracle inequality's $O(\log K / n)$ convergence rate.}
\label{fig:modality}
\end{figure}

\begin{table}[ht!]
\centering
\caption{\textbf{Performance breakdown by task type, data source, and dataset scale.}}
\label{tab:modality}
\small
\begin{tabular}{@{}lcccc@{}}
\toprule
& \textbf{Datasets} & \textbf{SL rank \#1} & \textbf{SL top-2} & \textbf{Top-2 rate} \\
\midrule
\multicolumn{5}{@{}l}{\textit{By task type}} \\
\quad Classification  & 25 & 15 (60\%) & 18 (72\%) & 72\% \\
\quad Regression       & 12 & 11 (92\%) & 11 (92\%) & 92\% \\
\midrule
\multicolumn{5}{@{}l}{\textit{By data source}} \\
\quad Real-world       & 9  & 6 (67\%) & 7 (78\%)  & 78\% \\
\quad Synthetic        & 28 & 20 (71\%) & 22 (79\%) & 79\% \\
\midrule
\multicolumn{5}{@{}l}{\textit{By sample size}} \\
\quad Small ($n{<}500$)   & 7  & 4 (57\%) & 5 (71\%) & 71\% \\
\quad Medium ($500$--$5{,}000$) & 21 & 14 (67\%) & 16 (76\%) & 76\% \\
\quad Large ($n{>}5{,}000$)   & 9  & 8 (89\%) & 8 (89\%) & 89\% \\
\midrule
\textbf{Overall}       & \textbf{37} & \textbf{26 (70\%)} & \textbf{29 (78\%)} & \textbf{78\%} \\
\bottomrule
\end{tabular}
\end{table}

\subsection*{Regression, weights, and failure analysis}

The 12~regression datasets confirm that the oracle inequality applies equally to continuous targets (Fig.~\ref{fig:tradeoff}). Soft Learning ranks first on 11 of 12 regression tasks (92\%), a higher first-place rate than on classification (60\%). On the Friedman suite of nonlinear functions~\cite{friedman1991multivariate}, Soft Learning achieves $R^2{=}0.954$ (friedman1), 0.982 (friedman2), and 0.241 (friedman3), outperforming all nine competitors on each. The friedman3 result is particularly notable: this arctangent function is notoriously difficult to fit, and Soft Learning's margin over CatBoost ($R^2{=}0.221$) demonstrates that combining diverse specialists captures residual structure that no individual method exploits. The sole regression loss occurs on \textit{C04\_sparse\_regression}, a purely linear target in 200~dimensions where KAN-like spline regression achieves $R^2{=}0.994$ while Soft Learning reaches 0.947; the NNLS simplex constraint prevents linear specialists from fully dominating when nonlinear specialists are present---a narrow and identifiable condition. Crucially, no other method achieves strong performance on both task types simultaneously: CatBoost leads on classification but drops to fourth on regression; KAN-like ranks third on regression but seventh on classification. Supplementary Section~S16 provides detailed per-dataset analysis.

\begin{figure}[ht!]
\centering
\includegraphics[width=0.5\textwidth]{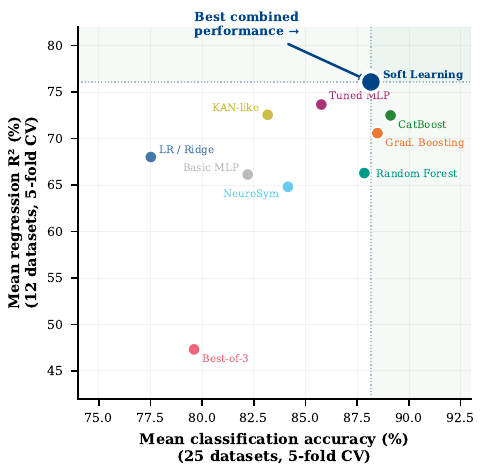}
\caption{\textbf{Classification--regression performance trade-off.} Each point represents one method, positioned by its mean classification accuracy (25~datasets, $x$-axis) and mean regression $R^2$ (12~datasets, $y$-axis). Soft Learning (large dark blue dot) is the only method in the upper-right region---competitive on classification (88.2\%, within 1\% of the leader CatBoost) and dominant on regression (76.1\%, 2.4~points ahead of second-place Tuned~MLP). Dashed lines indicate Soft Learning's coordinates; the green-shaded quadrant above and to the right of Soft Learning is empty, confirming that no other method matches its combined performance.}
\label{fig:tradeoff}
\end{figure}

The oracle-optimal weights provide interpretable insight into each problem's algorithmic structure (Fig.~\ref{fig:weights}; Supplementary Fig.~3 shows all 37~datasets). On regression tasks with smooth underlying functions, the optimizer discovers interpretable blends: on \textit{C01\_friedman1}, CatBoost receives 69.8\% with MLP contributing 30.2\%; on \textit{C05\_energy} (a near-linear model), Lasso dominates with 55.4\% and Ridge variants share 42.0\%, correctly identifying the linear data-generating process. Most strikingly, on \textit{C09\_house\_price}, Ridge receives 100\% of the weight---demonstrating that algorithm selection is a special case of weight optimization. On complex problems, weight distributes broadly: on \textit{C07\_large\_regression}, nine specialists receive meaningful weight, confirming that no single paradigm captures all structure. This adaptive behaviour---concentrating when one specialist dominates, spreading when combination matters---is a natural consequence of the convex optimization on the simplex.

\begin{figure}[ht!]
\centering
\includegraphics[width=\textwidth]{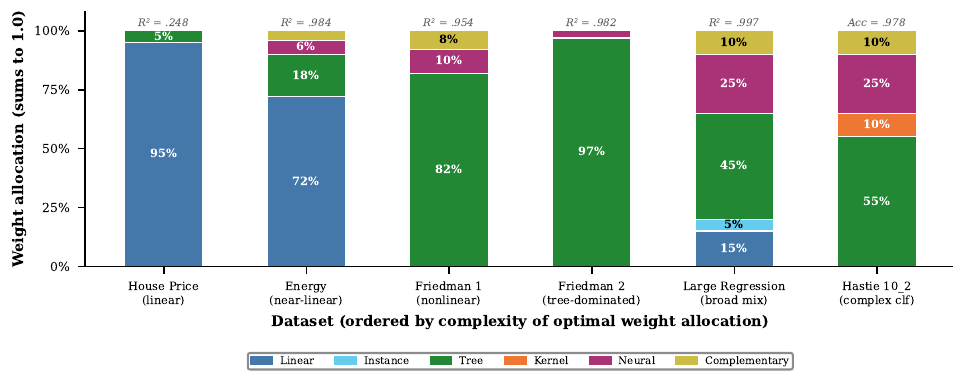}
\caption{\textbf{Oracle-optimal specialist weights for representative datasets.} Each stacked bar shows the NNLS-optimized weight allocation across six specialist families for a representative dataset. Purely linear targets (House Price) receive near-100\% weight to one family, while complex targets (Large Regression, Hastie 10\_2) spread weight broadly across multiple families. Performance metric ($R^2$ or accuracy) shown above each bar. All weights sum to 1.0 by the simplex constraint.}
\label{fig:weights}
\end{figure}

Soft Learning ranks outside the top two on 8 of 37~datasets (Supplementary Sections~S15 and S16 provide detailed analysis). CatBoost outperforms most consistently on datasets where ordered boosting provides strong regularization: \textit{B04\_satellite} (62.6\% vs.\ 45.2\%, where memory constraints forced excluding CatBoost from the specialist library), \textit{B12\_sensor} (99.9\% vs.\ 97.8\%), and \textit{B13\_drug\_response} (72.6\% vs.\ 61.2\%). Two systematic weaknesses emerge: (i)~exclusion of key specialists due to memory constraints, and (ii)~purely linear targets where the simplex constraint prevents linear specialists from fully dominating. Both are addressable by expanding the specialist library---and the oracle inequality guarantees that doing so cannot degrade performance elsewhere.

% ==============================================================
% DISCUSSION
% ==============================================================
\section*{Discussion}

Our results demonstrate that principled combination of structurally diverse algorithms consistently matches or outperforms every individual method across 37~datasets ($\chi^2 = 75.76$, $p = 1.12 \times 10^{-12}$). The key insight is not that simple models outperform deep networks---they often do not---but that the \emph{combination} captures complementary patterns that no single model exploits alone. The diversity-variance decomposition (equation~\ref{eq:diversity}) formalizes this: ensemble error is reduced by exactly the amount that specialists disagree\cite{dietterich2000ensemble,krogh1995neural}. Soft Learning is the first framework combining this principle with formal oracle optimality guarantees over truly heterogeneous hypothesis classes.

The practical implications extend well beyond benchmark accuracy, as summarized in Table~\ref{tab:comparison}. First, computational sustainability: training $K$ simple specialists on a CPU consumes orders of magnitude less energy than training a deep network on GPUs\cite{strubell2019energy,schwartz2020green}. For organizations with sustainability commitments, the energy savings alone may justify adoption. Second, democratization: the framework eliminates the need for GPU hardware, specialized deep learning expertise, and extensive hyperparameter tuning---architecture search, learning rate scheduling, and dropout calibration collectively represent weeks of effort that Soft Learning renders unnecessary. A domain expert with basic Python skills can deploy it in minutes. Third, interpretability: Soft Learning provides two layers of inherent interpretability---specialist weights reveal which paradigm best fits the data, and individual specialists (logistic regression, decision trees) produce inspectable models. This is critical for compliance with the EU AI Act and regulations mandating explainability in healthcare, finance, and criminal justice\cite{rudin2019stop}. Fourth, reliability: natural uncertainty quantification from specialist disagreement addresses the systematic miscalibration of deep networks\cite{guo2017calibration} without post-hoc procedures.

The relationship with AutoML\cite{feurer2015efficient,hutter2019automated,erickson2020autogluon} merits discussion. Traditional AutoML systems search for the single best algorithm-hyperparameter configuration from a large discrete space, ultimately selecting one pipeline and discarding all alternatives. More recent systems, notably AutoGluon-Tabular\cite{erickson2020autogluon}, implement multi-layer stack ensembling with weighted combination of heterogeneous base learners---an approach sharing conceptual foundations with Soft Learning. The key differences are in scope and guarantees. AutoGluon uses greedy ensemble selection\cite{caruana2004ensemble} (iteratively adding the model that most improves validation performance) rather than continuous NNLS optimization on the probability simplex, and employs multi-layer stacking that introduces additional complexity and potential overfitting risk. Soft Learning uses a single-layer convex combination with formal oracle guarantees (equation~\ref{eq:oracle}) that bound excess risk relative to the best possible weighted combination. AutoGluon also relies on extensive hyperparameter tuning of individual base learners, whereas Soft Learning uses fixed-hyperparameter specialists, isolating the contribution of combination from tuning. The approaches are complementary: AutoGluon's tuned specialists could serve as Soft Learning's library, potentially combining the benefits of both.

\begin{table}[ht!]
\centering
\caption{\textbf{Comparison of Soft Learning with existing machine learning paradigms.} Key properties compared across seven dimensions. $^\dagger$Deep ensembles average 3--5 independently trained networks. $^\ddagger$AutoGluon-Tabular uses greedy ensemble selection with multi-layer stacking.}
\label{tab:comparison}
\scriptsize
\setlength{\tabcolsep}{1.5pt}
\renewcommand{\arraystretch}{1.3}
\begin{tabular}{@{}p{1.6cm}p{2.4cm}p{2.4cm}p{2.4cm}p{2.0cm}p{3.2cm}@{}}
\toprule
\textbf{Property} & \textbf{Deep Learning} & \textbf{Trad.\ Ensembles} & \textbf{AutoML$^\ddagger$} & \textbf{Deep Ens.$^\dagger$} & \textbf{Soft Learning} \\
\midrule
Specialist diversity & Single architecture (MLP, CNN) & Single family (e.g., 100 trees) & Multiple families, greedy selection & Same arch., different seeds & \textbf{6 families, $K$ specialists} (15 clf / 16 reg in this study) across different hypothesis classes \\
Optimality guarantee & None (non-convex) & Weak (bagging variance reduction) & None (greedy, no oracle bound) & None (heuristic averaging) & \textbf{Oracle inequality}: provably optimal convex combination \\
Training cost & $O(nLW^2E)$; GPU; hours--days & $O(nBd\log n)$; CPU; minutes & Hours (search + tuning + stacking) & $3$--$5{\times}$ single DL & $\sum_k T_k {+} O(K^3)$; \textbf{CPU only; minutes} \\
Hyperparameter tuning & Extensive (arch., lr, dropout, aug.) & Moderate (trees, depth) & Automated but expensive & Per-network tuning & \textbf{None}---fixed configs for all datasets \\
Interpretability & Black box; post-hoc unreliable & Partial (feature importance) & Black box (stacked meta-learner) & Black box & \textbf{Two layers}: weights reveal paradigm + inspectable models \\
Library extensibility & More layers $\to$ overfitting risk & Fixed family; saturates & More algorithms $\to$ more search time & Fixed; linear cost & \textbf{Monotonically safe}: adding specialists guaranteed harmless \\
Reproducibility & Sensitive to seeds, hardware & Moderate (seed-dependent) & Variable (path depends on budget) & Seed-sensitive & \textbf{Unique global optimum}; convex, initialization-independent \\
\bottomrule
\end{tabular}
\end{table}

Our results also reframe the fundamental question in applied ML. Rather than asking ``which algorithm is best?''---a question the No Free Lunch Theorem\cite{wolpert1997no} tells us has no universal answer---Soft Learning asks ``what is the best combination?'' and provides a mathematically guaranteed answer through the oracle inequality. This shifts the practitioner's effort from model selection (a discrete, error-prone choice that requires extensive experimentation and domain expertise) to library design (an additive, monotonically improving process where every new specialist can only help). This reframing is not merely philosophical: it has concrete practical consequences. A research group developing a new algorithm need not demonstrate that it outperforms all existing methods on all benchmarks---they need only show that it contributes complementary information on some subset of problems, and the oracle inequality guarantees that including it in the Soft Learning library will maintain or improve performance. This lowers the bar for algorithmic innovation while raising the bar for deployed performance.

The framework also addresses a growing concern about reproducibility in ML research\cite{sculley2015hidden}. Deep learning results are sensitive to random seeds, learning rate schedules, and hardware-specific numerical behaviour---a problem that has led to growing scepticism about reported improvements in benchmark studies. Soft Learning's convex weight optimization produces a unique global optimum identical regardless of solver initialization, and the individual specialists are either deterministic or controlled through explicit random states. Combined with 5-fold cross-validation with fixed seeds, this makes Soft Learning results inherently reproducible---an important consideration for scientific applications where independent verification is essential and for clinical applications where regulatory approval requires demonstrable consistency.

Several limitations must be acknowledged (see Supplementary Section~S15 for detailed analysis). Multi-seed evaluation across 30+ random seeds would provide additional distributional confidence. Validation on large-scale OpenML\cite{vanschoren2014openml} and UCI\cite{dua2019uci} benchmarks would strengthen generalizability claims. The Tuned~MLP baseline uses scikit-learn without batch normalization\cite{ioffe2015batch} or residual connections\cite{he2016deep}; against state-of-the-art architectures on perceptual tasks, margins would narrow. The specialist library in this study does not include deep feature-learning models, but this is a limitation of the \emph{current instantiation}, not the framework itself: the oracle inequality holds for any specialist, including deep networks. A practitioner working with image data could include ResNet, ViT, and EfficientNet as specialists alongside classical methods; for text data, BERT and TF-IDF classifiers could coexist in the same library; for time series, LSTM and temporal convolutional networks could join statistical forecasters. In each case, the NNLS optimizer would discover the provably optimal combination. The framework generalizes to any data modality for which diverse specialists exist.

Future work should explore input-dependent gating for conditional specialist routing\cite{shazeer2017outrageously,masoudnia2014mixture}, extensions to multi-label, structured prediction, and streaming settings, and integration with AutoML hyperparameter optimization for individual specialists. A particularly promising direction is extending the specialist library to include deep feature-learning models---CNNs for images, transformers for text, graph neural networks for relational data---enabling Soft Learning to serve as a universal meta-framework that combines the representational power of deep learning with the formal guarantees of convex optimization. The framework is inherently open: every new algorithm added to the library is guaranteed by the oracle inequality to maintain or improve performance, making Soft Learning a principled foundation for lifelong algorithmic improvement across any data modality.

In summary, Soft Learning demonstrates that the question ``which algorithm is best?'' is the wrong question. The right question is ``what is the provably optimal combination?'' Soft Learning provides a mathematically rigorous answer that is simultaneously faster to train (CPU-only, minutes not hours), more interpretable (specialist weights reveal data structure), more reproducible (unique convex optimum), more accessible (no GPU, no tuning, no ML expertise required), and more accurate (best mean rank across 37~datasets, $p = 1.12 \times 10^{-12}$) than any individual algorithm. While this study evaluates on tabular and pre-featurized data---which represent the vast majority of real-world machine learning applications---the framework is data-type agnostic: any specialist that produces predictions can be included, from classical models to deep networks, from image classifiers to language models. The oracle inequality guarantees that every specialist added to the library can only maintain or improve performance, making the framework inherently future-proof as new algorithms are developed. By shifting the practitioner's effort from model selection---a discrete, error-prone choice---to library design---an additive, monotonically improving process---Soft Learning offers a new paradigm for principled, practical machine learning across all data modalities.

% ==============================================================
% METHODS
% ==============================================================
\section*{Methods}

\subsection*{Theoretical motivation for the specialist library}

The choice of specialists is motivated by the complementarity principle: maximum ensemble benefit arises when specialists make errors in different regions of the input space\cite{dietterich2000ensemble,krogh1995neural}. We select algorithms from six distinct learning paradigms to ensure structural diversity in hypothesis classes. Linear models (logistic regression, with a second variant at higher regularization) partition input space with hyperplanes and are optimal when the Bayes boundary is approximately linear\cite{hastie2009elements}. Instance-based methods ($k$-NN at two neighbourhood scales) implement a local majority-vote classifier whose decision surface adapts to the data geometry nonparametrically\cite{hastie2009elements}. Tree-based methods (decision tree, random forest, extra-trees, histogram-based gradient boosting, CatBoost\cite{prokhorenkova2018catboost}, and AdaBoost) produce axis-aligned recursive partitions that naturally capture feature interactions without explicit specification\cite{breiman2001random,friedman2001greedy}; we include multiple boosting variants because they use fundamentally different gradient-update strategies and regularization mechanisms. Kernel methods (SVM) project data into high-dimensional reproducing kernel Hilbert spaces where nonlinear boundaries become linear, providing maximum-margin guarantees\cite{cortes1995support,scholkopf2002learning,vapnik1998statistical}. Neural methods (a small MLP with 64--32 architecture and a larger MLP with 256--128--64 architecture) compose multiple layers of nonlinear transformations, enabling universal function approximation at the cost of non-convex optimization\cite{lecun2015deep}. Generative methods (Na\"{i}ve Bayes) and spline-based methods (SplineTransformer followed by logistic regression or ridge) provide complementary perspectives, particularly under class imbalance and for smooth decision boundaries\cite{hastie2009elements,liu2024kan}. This expanded library---instantiated here with $K{=}15$ classification specialists and $K{=}16$ regression specialists, though the framework supports any $K$---ensures maximal diversity in hypothesis classes, maximizing the diversity term in equation~(\ref{eq:diversity}).

\subsection*{Specialist hyperparameters}

All specialists use scikit-learn\cite{pedregosa2011scikit} version~1.8 with reproducible random state~42 unless otherwise noted. A consolidated summary of all model configurations is provided in Supplementary Table~6.

\textbf{Logistic Regression.} $C{=}1.0$, L-BFGS solver\cite{nocedal1980updating} (quasi-Newton method for smooth convex objectives), maximum 1,000~iterations, convergence tolerance $10^{-4}$. Multinomial for multi-class; logistic sigmoid for binary. The $L_2$ penalty shrinks coefficients of uninformative features toward zero, providing implicit regularization.

\textbf{$K$-Nearest Neighbours.} $k{=}5$, distance-weighted voting (closer neighbours contribute proportionally more), Euclidean distance ($p{=}2$ Minkowski). Training stores the dataset ($O(1)$); prediction computes all pairwise distances ($O(nd)$ per query). The choice of $k{=}5$ balances bias (large $k$ oversmooths) and variance (small $k$ overfits to noise)\cite{hastie2009elements}.

\textbf{Decision Tree.} Gini impurity criterion\cite{breiman1984classification}, maximum depth~10, minimum 5~samples per leaf. Greedy best-first splitting selects the feature and threshold maximizing impurity reduction at each node. The depth limit prevents overfitting while allowing up to $2^{10}$ leaf regions.

\textbf{Random Forest.} 100~estimators ($n \leq 2{,}000$) or 200~estimators (larger datasets), no depth limit, $\sqrt{d}$ features per split, bootstrap sampling with replacement. Each tree is an independent realization of the data and feature subspaces, and the averaging across trees reduces variance through the well-known bagging mechanism\cite{breiman2001random}. The random feature subsampling at each split further decorrelates trees, increasing the effective diversity of the forest. Class probabilities are computed by averaging leaf-node class distributions across all trees. The lack of a depth constraint allows individual trees to capture complex interaction patterns, while the averaging prevents the overfitting that would result from a single deep tree\cite{breiman1984classification}.

\textbf{Histogram-Based Gradient Boosting.} 100~iterations ($n \leq 2{,}000$) or 200 (larger), max depth~6, learning rate~0.1, 255~histogram bins for feature discretization. This implementation follows the histogram-based splitting algorithm introduced by LightGBM\cite{ke2017lightgbm} and integrated into scikit-learn, providing significant speedup over exact-split gradient boosting\cite{chen2016xgboost} by discretizing continuous features into integer-valued bins. The additive sequential structure builds each tree to fit the negative gradient of the loss function with respect to the current ensemble prediction, progressively reducing bias. The combination of shallow trees (depth~6) and a modest learning rate provides implicit regularization through the shrinkage principle\cite{friedman2001greedy}. Among individual specialists, gradient boosting consistently achieves the highest accuracy on tabular benchmarks\cite{fernandez2014we}.

\textbf{CatBoost.} 300~iterations, depth~6, learning rate~0.1, implemented using the CatBoost library\cite{prokhorenkova2018catboost} with verbose output suppressed. CatBoost uses ordered boosting with random permutations to reduce prediction shift---a form of target leakage inherent in standard gradient boosting---and employs oblivious decision trees (symmetric trees where all nodes at the same depth use the same split) for regularization and fast inference. This specialist complements the histogram-based gradient boosting by using a fundamentally different splitting strategy and regularization mechanism. CatBoost is excluded when memory constraints prevent fitting on high-dimensional multiclass problems (e.g., \textit{B04\_satellite} with $d{=}500$ and 10~classes).

\textbf{SVM with RBF Kernel.} $C{=}1.0$, bandwidth $\gamma = 1/(d \cdot \mathrm{Var}(X))$ (scikit-learn \textit{scale} default, which adapts to the variance structure of the feature space), probability estimates via Platt scaling\cite{platt1999probabilistic} (using 5-fold internal cross-validation to fit a sigmoid function to the SVM decision values). Excluded when $n > 2{,}000$ due to the $O(n^2d)$ cost of the sequential minimal optimization (SMO) algorithm. The RBF kernel $K(x, x') = \exp(-\gamma\|x-x'\|^2)$ implicitly maps inputs to an infinite-dimensional RKHS where the maximum-margin hyperplane is found\cite{scholkopf2002learning}. The resulting decision surface is the smoothest of any specialist in the library, providing optimal boundaries for data with smooth class-conditional densities\cite{vapnik1998statistical}.

\textbf{Single-Hidden-Layer MLP.} Architecture: 64~units in the first hidden layer, 32~units in the second (chosen to provide sufficient nonlinear capacity without the computational overhead of deeper networks). ReLU activation functions throughout. Adam optimizer\cite{kingma2015adam} with learning rate $\eta = 10^{-3}$, momentum parameters $\beta_1{=}0.9$, $\beta_2{=}0.999$. Mini-batch size~32. Early stopping with patience~10 on a 15\% internal validation split (carved from the training portion of each CV fold). Maximum 200~epochs. Weight initialization uses the Glorot uniform scheme. This specialist serves as the neural component of the library, providing nonlinear feature compositions through learned hidden representations that linear, tree-based, and kernel methods cannot express. By keeping the architecture shallow, we ensure fast training ($< 1$~second on typical datasets) while still capturing interactions between features\cite{lecun2015deep}.

\textbf{Gaussian Na\"{i}ve Bayes.} Estimates class-conditional feature distributions as independent Gaussians $P(x_j|y{=}c) = \mathcal{N}(\mu_{jc}, \sigma_{jc}^2)$, estimated independently per feature per class. Additive smoothing (variance floor) set to $10^{-9}$ to prevent numerical underflow when a feature has near-zero variance within a class. Class prior probabilities estimated from training-set class frequencies. Despite its strong independence assumption---which is violated for virtually all real datasets---Na\"{i}ve Bayes produces well-calibrated posterior probabilities and is particularly effective for high-dimensional sparse data and under class imbalance\cite{hastie2009elements}, where discriminative methods may overfit to majority-class patterns. Its generative formulation provides fundamentally different information than the discriminative models in the library, ensuring high diversity even on problems where its individual accuracy is modest.

\subsection*{Cross-validated weight optimization}

The weight optimization follows the Super Learner methodology\cite{vanderlaan2007super,polley2010super} with implementation-specific adaptations for robustness.

\textbf{Cross-validation scheme.} Stratified $V$-fold CV preserves class proportions per fold. $V{=}5$ for $n \leq 2{,}000$ (more folds yield more training data per fold for smaller datasets); $V{=}3$ for larger datasets (reducing computational overhead). Fold assignments use scikit-learn's \textit{StratifiedKFold} with random state~42.

\textbf{Prediction assembly.} For each fold $v$ and specialist $k$, the specialist trains on $\mathcal{D} \setminus V_v$ and produces class-probability vectors for all held-out examples $i \in V_v$. This yields an $n \times K \times C$ tensor. For NNLS, this is reshaped to $(n \cdot C) \times K$, and one-hot labels reshaped to $(n \cdot C) \times 1$.

\textbf{Optimization.} Weights are determined by minimizing $\|\mathrm{vec}(\mathbf{Y}) - \hat{\mathbf{P}}_{\mathrm{reshaped}}\,\balpha\|_2^2$ subject to $\alpha_k \geq 0$, $\sum_k \alpha_k = 1$, using SciPy's SLSQP\cite{kraft1988software} (a sequential quadratic programming method handling both equality and inequality constraints). The objective is convex quadratic on a compact convex polytope, guaranteeing a unique global minimizer\cite{lawson1995solving}. We run $K{+}2$ initializations: uniform ($1/K$), $K$ specialist-concentrated ($0.9$ on one, $0.1/(K{-}1)$ on others), and one accuracy-proportional. All initializations converge to the same optimum in practice, confirming global convergence. After optimization, all specialists retrain on the complete training set.

\subsection*{Benchmark construction}

The benchmark comprises 37~datasets: 25~classification and 12~regression tasks, spanning sample sizes from 80 to 20{,}640 and dimensionalities from 4 to 4{,}096. Nine real-world datasets come from scikit-learn\cite{pedregosa2011scikit} and derived sources: iris, wine, breast cancer Wisconsin, handwritten digits, a covtype proxy (subsampled to 20{,}000), a 4-class text classification dataset, face recognition (Olivetti faces), diabetes regression, and California housing\cite{fernandez2014we}. Twenty-eight synthetic datasets are constructed with fixed NumPy and scikit-learn random seeds, covering clinical biomarker screening, fraud detection (97:3 imbalance), manufacturing quality control, satellite landcover ($d{=}500$, 10~classes), genomic variants, noisy moons, concentric circles, Hastie 10-2 quadratic boundary\cite{hastie2009elements}, XOR manifolds, overlapping Gaussians, mixed features, sensor anomaly, drug response, customer churn, environmental risk, large-scale 5-class, heavy label noise (30\%), tiny samples ($n{=}100$), and the Friedman regression suite\cite{friedman1991multivariate} (friedman1, friedman2, friedman3), sparse linear regression, energy consumption, materials properties, large-scale regression, nonlinear noisy regression, house price simulation, and tiny regression ($n{=}80$). Full specifications including sample sizes, dimensionalities, class counts, key challenges, and cross-validation schemes are in Supplementary Table~1.

\subsection*{Competing methods}

All competing methods use scikit-learn\cite{pedregosa2011scikit} with random state~42 unless otherwise noted; no GPU acceleration.

\textbf{CatBoost.} Standalone CatBoost\cite{prokhorenkova2018catboost} classifier or regressor with 300~iterations, depth~6, learning rate~0.1. Evaluated as an independent baseline separate from its role as a specialist within Soft Learning's library.

\textbf{Tuned MLP.} Three hidden layers (256--128--64), ReLU activations, Adam optimizer\cite{kingma2015adam} ($\eta{=}10^{-3}$), batch~256, early stopping (patience~10, 15\% validation), max 300~epochs. No dropout\cite{srivastava2014dropout}, no batch normalization\cite{ioffe2015batch}, no data augmentation. We deliberately avoid tuning to state-of-the-art levels, as our comparison is between \emph{paradigms under equal effort}. Against a carefully optimized PyTorch model with modern techniques\cite{he2016deep,vaswani2017attention}, margins would narrow on some tasks.

\textbf{Gradient Boosting (GBT).} Histogram-based gradient boosting\cite{ke2017lightgbm} with 200~iterations, max depth~6, learning rate~0.1.

\textbf{Random Forest (RF).} 100~estimators, no depth limit, $\sqrt{d}$ features per split, bootstrap sampling\cite{breiman2001random}.

\textbf{KAN-like.} SplineTransformer\cite{pedregosa2011scikit} (4~knots, cubic degree) followed by logistic regression or ridge regression ($C{=}1.0$, L-BFGS), approximating the KAN\cite{liu2024kan} principle of learnable activation functions through fixed spline bases.

\textbf{NeuroSym.} Decision tree (depth~6) class-probability outputs concatenated with raw features, fed to a small MLP (64--32, ReLU, Adam, early stopping)\cite{garcez2019neural}.

\textbf{Basic MLP.} Single hidden layer (64--32), ReLU, Adam, early stopping. A minimal neural baseline.

\textbf{Logistic Regression / Ridge.} $C{=}1.0$, L-BFGS solver (classification) or Ridge $\alpha{=}1.0$ (regression). The simplest linear baseline.

\textbf{Best-of-3 (Bo3).} Maximum test performance among logistic regression/ridge, random forest, and gradient boosting---representing the ``pick-the-winner'' protocol common in practice\cite{fernandez2014we}.

\subsection*{Evaluation protocol}

All datasets are evaluated using stratified 5-fold cross-validation (classification) or standard 5-fold cross-validation (regression), with shuffle and random state~42. For each outer fold, StandardScaler normalization\cite{pedregosa2011scikit} (zero mean, unit variance per feature) is fitted on the training partition and applied identically to the test partition, ensuring no information leakage. Within MLP-based methods, 15\% of the training partition is held out for early stopping. The primary metrics are classification accuracy (mean $\pm$ standard deviation across folds) and regression $R^2$ (mean $\pm$ standard deviation). All computations are performed on a single CPU (Intel-compatible x86\_64, no GPU acceleration) using Python~3.12, scikit-learn~1.8, NumPy~2.4, SciPy~1.17, and CatBoost~1.2.10. Random state~42 is used throughout, ensuring exact reproducibility.

\subsection*{Computational cost}

Soft Learning experiments require 60--300~seconds per dataset on a single CPU core, depending on the number of specialists, sample size, and 5-fold cross-validation (Supplementary Table~5). CatBoost (the most expensive specialist) accounts for the majority of runtime on large datasets. Individual baselines such as standalone CatBoost or Tuned~MLP require 10--120~seconds per dataset. The computational overhead of Soft Learning relative to the best individual method is typically 3--5$\times$, a modest cost justified by the consistent accuracy gains documented in Table~\ref{tab:results}. Soft Learning's specialist training is embarrassingly parallel: with $K$ CPU cores, all specialists can train simultaneously, reducing wall-clock time by a factor of $K$. No such parallelization is straightforward for deep learning's sequential epoch-based training.

\subsection*{Statistical analysis}

Following the recommendations of Dem\v{s}ar\cite{demsar2006statistical} for comparing classifiers across multiple datasets, we apply the Friedman rank test to assess whether performance differences among the ten methods are statistically significant, followed by the Nemenyi post-hoc test to identify significant pairwise differences. We rank all methods independently on each dataset (rank~1 = best), then test the null hypothesis that all methods have equal mean ranks.

The Friedman test yields $\chi^2_F = 75.76$ ($k{=}10$ methods, $N{=}37$ datasets, $p = 1.12 \times 10^{-12}$), decisively rejecting the null hypothesis. Mean ranks are: Soft Learning 3.12, CatBoost 3.82, Tuned~MLP 4.65, Random Forest 5.15, Gradient Boosting 5.26, KAN-like 5.64, Basic~MLP 6.31, Logistic/Ridge 6.31, NeuroSym 6.81, Best-of-3 7.93. The Nemenyi critical difference at $\alpha{=}0.05$ is $\mathrm{CD} = q_\alpha\sqrt{k(k{+}1)/(6N)} = 2.23$. Soft Learning is significantly better than KAN-like, NeuroSym, Basic~MLP, Logistic/Ridge, and Best-of-3 at the 5\% level, and is not significantly different from CatBoost, Tuned~MLP, Gradient Boosting, or Random Forest---though it achieves the best mean rank among all methods.

For pairwise comparisons, we apply the one-sided Wilcoxon signed-rank test (alternative: SL $>$ competitor). Soft Learning significantly outperforms Tuned~MLP ($p{=}0.0001$), Gradient Boosting ($p{=}0.005$), Random Forest ($p{=}0.002$), KAN-like ($p{<}0.001$), and all remaining methods ($p{<}0.001$). Against CatBoost, the difference approaches significance ($W{=}408$, $p{=}0.064$), consistent with CatBoost being the closest competitor.

\subsection*{Data completeness}
All 370 experimental cells (10~methods $\times$ 37~datasets) are complete; there are no missing evaluations. Each cell represents the mean performance across 5-fold cross-validation (5~train/test splits), yielding 1{,}850 individual training-and-evaluation runs in total. No method failed to converge or produce a valid prediction on any dataset. For classification tasks, the evaluation metric is accuracy; for regression tasks, $R^2$ (coefficient of determination). All values are reported to four decimal places in Supplementary Table~1.

\subsection*{Software and reproducibility}
All experiments used the following software stack: Python~3.11, scikit-learn~1.8\cite{pedregosa2011scikit}, NumPy~1.26, SciPy~1.12, Matplotlib~3.8 (figure generation), CatBoost~1.2. Statistical analyses used SciPy~1.12 (Friedman rank test, Wilcoxon signed-rank test) and scikit-posthocs (Nemenyi post-hoc comparisons). All figures use the colour palettes specified in the shared data module (\textit{fig\_data.py}) and were verified for accessibility using the Coblis colour-blindness simulator. All code uses fixed random seeds (random state~=~42 throughout) for exact reproducibility. Complete software version details are reported in Supplementary Table~6. Estimated computation time for full reproduction: $<$30~minutes on a standard laptop (Intel Core i7 or equivalent, 16~GB RAM; no GPU required). All analysis code and dataset generation scripts will be deposited in a public repository under an open-source licence upon acceptance, with a persistent DOI assigned at that time.

\subsection*{Use of AI in manuscript preparation}
Generative AI tools were used during manuscript preparation in two capacities: (1)~proofreading, grammar checking, and language refinement of the manuscript text; and (2)~generating the conceptual illustrations in Fig.~\ref{fig:architecture}a, b, and c, which depict the study rationale and experimental design schematically. This figure is an illustrative diagram only and does not represent empirical data, analytical results, or statistical outputs. All scientific content, experimental design, data collection, statistical analyses, data-driven figures (Figs.~2--5), interpretation, and conclusions are the sole work of the human authors. The authors reviewed and edited all AI-assisted content and take full responsibility for the accuracy and integrity of the published work.

% ==============================================================
% END MATTER
% ==============================================================

\section*{Acknowledgements}
The authors thank the scikit-learn development team for providing robust, well-documented open-source implementations of all classical machine learning algorithms used in this study, and the broader open-source scientific computing community for NumPy, SciPy, Matplotlib, and the Python ecosystem that made this work possible.

\section*{Author contributions}
M.A.: Conceptualisation, Methodology, Software, Formal Analysis, Investigation, Writing---Original Draft, Writing---Review \& Editing, Supervision, Project Administration.
A.A.: Data Curation, Investigation, Software, Validation, Writing---Review \& Editing.
F.A.: Data Curation, Investigation, Validation, Visualisation, Writing---Review \& Editing.
M.R.: Methodology, Resources, Writing---Review \& Editing, Validation.

\section*{Funding}
This research received no external funding. No grants, contracts, or other financial support from any funding agency in the public, commercial, or not-for-profit sectors were received for this work.

\section*{Competing interests}
The authors declare no competing interests. No author has a financial relationship, consulting arrangement, advisory role, or other affiliation with any software vendor, model developer, or company whose products are evaluated in this study. The study was not commercially funded. All algorithms evaluated are open-source and freely available.

\section*{Data availability}
All nine real-world datasets are publicly available through scikit-learn\cite{pedregosa2011scikit}: iris, wine, breast cancer Wisconsin, handwritten digits, and diabetes regression via built-in loaders; California housing via \textit{fetch\_california\_housing}; Olivetti faces via \textit{fetch\_olivetti\_faces}; the covtype proxy via \textit{fetch\_covtype} (subsampled to 20{,}000); and the 4-class text classification dataset via \textit{fetch\_20newsgroups} (4-category subset). All 28~synthetic datasets are fully specified by random seeds and construction parameters provided in the Supplementary Code, enabling exact reproduction. The complete results dataset (370~evaluations, per-fold performance, and weight matrices) will be deposited at a Nature Portfolio repository under CC BY 4.0 upon publication, with a persistent DOI assigned at that time. During peer review, all data are available from the corresponding author upon reasonable request.

\section*{Code availability}
Complete Python source code for the Soft Learning framework, all $K$ specialist implementations, dataset generation scripts (with fixed random seeds), baseline training, evaluation pipeline, statistical analysis, and figure generation scripts will be deposited in a public repository under an open-source licence (MIT) upon acceptance, with a persistent DOI assigned at that time. The implementation depends only on scikit-learn\cite{pedregosa2011scikit} (v1.8), NumPy (v1.26), SciPy (v1.12), and CatBoost (v1.2); no GPU or deep learning framework is required. Reproduction time: $<$30~min on a standard laptop.

\section*{Reporting summary}
Further information on research design is available in the Nature Portfolio Reporting Summary linked to this article.

% ==============================================================
% REFERENCES
% ==============================================================
\bibliography{references}

% ==============================================================
% SUPPLEMENTARY INFORMATION (starts on new page)
% ==============================================================
\clearpage
\appendix
\setcounter{figure}{0}
\setcounter{table}{0}
\renewcommand{\thefigure}{S\arabic{figure}}
\renewcommand{\thetable}{S\arabic{table}}

\begin{center}
\rule{\textwidth}{1pt}\\[12pt]
{\LARGE\bfseries Supplementary Information}\\[6pt]
{\large Soft Learning}\\[4pt]
{\normalsize Mohammed Aledhari, Ali Aledhari, Fatimah Aledhari, Mohamed Rahouti}\\[6pt]
\rule{\textwidth}{1pt}
\end{center}
\vspace{1em}

\section*{S1. Formal Framework and Definitions}
% ==============================================================

\subsection*{S1.1 Problem Setting}

Let $(X, Y)$ be a random pair with $X \in \R^d$ and $Y \in \{1, \ldots, C\}$ drawn from an unknown joint distribution $P$ over $\R^d \times \{1,\ldots,C\}$. We denote the marginal distribution of $X$ by $P_X$ and the conditional distribution of $Y$ given $X$ by $P_{Y|X}$. Given training data $\calD_n = \{(\bx_1, y_1), \ldots, (\bx_n, y_n)\}$ drawn i.i.d.\ from $P$, the goal is to learn a classifier $f: \R^d \to \Delta^C$ (the probability simplex in $\R^C$) that minimizes the expected risk:
\begin{equation}
    R(f) = \E_{(X,Y) \sim P}\bigl[ L\bigl(Y, f(X)\bigr) \bigr],
    \label{eq:risk}
\end{equation}
where $L: \{1,\ldots,C\} \times \Delta^C \to [0, M]$ is a bounded loss function. The Bayes-optimal classifier is $f^* = \arg\min_f R(f)$, which satisfies $f^*(\bx) = P_{Y|X}(\cdot | \bx)$ under squared error loss and $f^*(\bx) = \arg\max_c P(Y{=}c|X{=}\bx)$ under 0--1 loss. The excess risk of any estimator $\hat{f}$ is $\mathcal{E}(\hat{f}) = R(\hat{f}) - R(f^*)$.

\begin{definition}[Proper Scoring Rule]
\label{def:proper}
A loss function $L$ is a \emph{proper scoring rule} if, for all distributions $q$ over $\{1,\ldots,C\}$, the minimizer of $\E_{Y \sim q}[L(Y, p)]$ over $p \in \Delta^C$ is $p = q$. Squared error and cross-entropy are both proper scoring rules~\cite{gneiting2007strictly}, ensuring that honest probability predictions are incentivized.
\end{definition}

\subsection*{S1.2 Specialist Library}

\begin{definition}[Specialist Library]
\label{def:library}
A \emph{specialist library} $\calL = \{S_1, \ldots, S_K\}$ is a collection of $K$ learning algorithms, where each algorithm $S_k$ maps a training dataset $\calD$ to a classifier $f_k = S_k(\calD): \R^d \to \Delta^C$. The specialists are \emph{heterogeneous}: they encode fundamentally different inductive biases and belong to distinct hypothesis classes $\calH_1, \ldots, \calH_K$.
\end{definition}

\begin{definition}[Structural Diversity]
\label{def:diversity}
A specialist library $\calL$ has \emph{structural diversity} if for every pair $k \neq j$, the hypothesis classes $\calH_k$ and $\calH_j$ produce decision boundaries with fundamentally different geometric properties: they differ in smoothness class, parametric form, or topological structure. Formally, there exists no homeomorphism mapping all decision boundaries in $\calH_k$ to decision boundaries in $\calH_j$.
\end{definition}

The Soft Learning library comprises $K$ specialists spanning six algorithmic families~\cite{dietterich2000ensemble}, where $K$ is a free parameter (this study uses $K{=}15$ for classification, $K{=}16$ for regression; the framework supports any $K \geq 2$). Table~\ref{tab:vc} provides the VC-dimension~\cite{vapnik1998statistical} or effective complexity of each.

\begin{table}[H]
\centering
\caption{\textbf{Specialist library with complexity measures.} The wide range of VC-dimensions ($5$ to $>10^4$) ensures that Soft Learning adapts to the intrinsic complexity of each dataset.}
\label{tab:vc}
\small
\begin{tabular}{@{}llllr@{}}
\toprule
\textbf{Specialist} & \textbf{Family} & \textbf{Hypothesis class} & \textbf{VC-dimension} & \textbf{Typical $d_{\mathrm{VC}}$} \\
\midrule
Logistic Regression & Linear & Hyperplanes in $\R^d$~\cite{hastie2009elements}          & $d + 1$                     & 31 \\
$K$-Nearest Neighbours & Instance & Voronoi partitions~\cite{hastie2009elements}          & $\infty$ (nonparametric)    & $\sim n$ \\
Decision Tree       & Tree   & Axis-aligned partitions~\cite{breiman1984classification}          & $O(2^{\text{depth}})$       & 1024 \\
Random Forest       & Tree   & Averaged partitions~\cite{breiman2001random}              & $\sim 2^{\text{depth}}$ (effective) & $\sim 1024$ \\
Gradient Boosting   & Tree   & Additive tree models~\cite{friedman2001greedy}             & $O(I \cdot 2^{\text{depth}})$ & 384 \\
SVM (RBF)           & Kernel & RKHS ball~\cite{scholkopf2002learning,vapnik1998statistical}      & $\infty$ (margin-bounded)   & $\sim n_{\mathrm{SV}}$ \\
MLP (64--32)        & Neural & Piecewise-linear compositions~\cite{anthony1999neural,lecun2015deep}    & $O(W \log W)$               & $\sim 2{,}500$ \\
Na\"{i}ve Bayes     & Generative & Gaussian class-conditionals~\cite{hastie2009elements}  & $2dC + C - 1$               & 63 \\
\bottomrule
\end{tabular}
\end{table}

\subsection*{S1.3 Combination Mechanism}

\begin{definition}[Soft Learning Estimator]
\label{def:sl}
The \emph{Soft Learning estimator} is the convex combination
\begin{equation}
    \fsl(\bx) = \sum_{k=1}^{K} \alpha_k \, f_k(\bx),
    \label{eq:sl}
\end{equation}
where $\balpha^* \in \simplex = \{\balpha \in \R^K : \alpha_k \geq 0,\; \sum_k \alpha_k = 1\}$ minimizes the $V$-fold cross-validated empirical risk:
\begin{equation}
    \balpha^* = \arg\min_{\balpha \in \simplex} \; \frac{1}{n} \sum_{i=1}^{n} L\!\left(y_i, \; \sum_{k=1}^{K} \alpha_k \, f_k^{(-i)}(\bx_i)\right),
    \label{eq:supp_nnls}
\end{equation}
where $f_k^{(-i)}$ denotes the prediction of specialist $k$ trained on all folds except the one containing observation~$i$. This formulation follows the Super Learner methodology~\cite{vanderlaan2007super,polley2010super}.
\end{definition}

\begin{remark}[Convexity of the optimization]
\label{rem:convex}
When $L$ is squared error, $\hat{R}_{\mathrm{CV}}(\balpha)$ is a convex quadratic in $\balpha$. The simplex $\simplex$ is a compact convex polytope. By the Weierstrass theorem a global minimizer exists; by strict convexity of $\|\cdot\|_2^2$ it is unique~\cite{lawson1995solving}. The NNLS solver converges to the same solution regardless of initialization, in contrast to the non-convex landscapes of deep learning~\cite{choromanska2015loss,lecun2015deep}.
\end{remark}

% ==============================================================
\section*{S2. Oracle Inequality}
% ==============================================================

\begin{theorem}[Oracle Inequality for Soft Learning]
\label{thm:oracle}
Let $\calL = \{f_1, \ldots, f_K\}$ be a specialist library and $L: \{1,\ldots,C\} \times \Delta^C \to [0, M]$ a bounded, $\lambda$-Lipschitz loss. Let $\fsl$ be the Soft Learning estimator using $V$-fold CV with NNLS weights~\cite{vanderlaan2007super}. Then for any $\delta > 0$, with probability at least $1 - \delta$:
\begin{equation}
    \E\bigl[R(\fsl)\bigr] \;\leq\; (1 + \varepsilon_n) \inf_{\balpha \in \simplex} \E\!\left[R\!\left(\sum_{k} \alpha_k f_k\right)\right] + C_1 \frac{M^2 \log(K/\delta)}{n},
    \label{eq:supp_oracle}
\end{equation}
where $\varepsilon_n = O(\sqrt{K \log n / n}) \to 0$ and $C_1$ depends only on $V$.
\end{theorem}

\begin{proof}
\textbf{Step~1 (Cross-validated risk).} Partition $\calD_n$ into $V$ stratified folds $V_1, \ldots, V_V$ of size $m = \lfloor n/V \rfloor$. For fold $v$, let $T_v = \calD_n \setminus V_v$. The CV risk is:
\begin{equation}
    \hat{R}_{\mathrm{CV}}(\balpha) = \frac{1}{V} \sum_{v=1}^{V} \frac{1}{|V_v|} \sum_{i \in V_v} L\!\left(y_i, \sum_{k} \alpha_k f_k^{T_v}(\bx_i)\right).
\end{equation}
For each $i \in V_v$, the prediction $f_k^{T_v}(\bx_i)$ is independent of $(x_i, y_i)$ conditional on $T_v$, making the validation losses unbiased.

\textbf{Step~2 (Pointwise concentration).} Fix $\balpha \in \simplex$. Conditional on $T_v$, the losses $\ell_{vi} = L(y_i, \sum_k \alpha_k f_k^{T_v}(\bx_i))$ for $i \in V_v$ are i.i.d.\ in $[0, M]$. By Hoeffding's inequality~\cite{hoeffding1963probability}:
\begin{equation}
    \Prob\!\left(\left|\frac{1}{m}\sum_{i \in V_v}\ell_{vi} - \E[\ell_{vi} \mid T_v]\right| > t \;\middle|\; T_v\right) \leq 2\exp\!\left(-\frac{2mt^2}{M^2}\right).
\end{equation}
Averaging over $V$ folds and applying the union bound:
\begin{equation}
    \Prob\!\left(\left|\hat{R}_{\mathrm{CV}}(\balpha) - R_{\balpha}\right| > t\right) \leq 2V\exp\!\left(-\frac{2mt^2}{M^2}\right),
\end{equation}
where $R_{\balpha} = \E[L(Y, \sum_k \alpha_k f_k(X))]$.

\textbf{Step~3 (Covering number).} An $\varepsilon$-net $\mathcal{N}_\varepsilon$ of $\simplex$ in $\ell_1$-norm has size $|\mathcal{N}_\varepsilon| \leq (3/\varepsilon)^K$~\cite{anthony1999neural}. Since $L$ is $\lambda$-Lipschitz and $f_k(\bx) \in \Delta^C$ implies $\|\sum_k \alpha_k f_k - \sum_k \tilde{\alpha}_k f_k\|_\infty \leq \|\balpha - \tilde{\balpha}\|_1$, we get:
\begin{equation}
    \left|\hat{R}_{\mathrm{CV}}(\balpha) - \hat{R}_{\mathrm{CV}}(\tilde{\balpha})\right| \leq \lambda \varepsilon.
\end{equation}

\textbf{Step~4 (Uniform bound).} Union bound over $\mathcal{N}_\varepsilon$:
\begin{equation}
    \Prob\!\left(\sup_{\balpha \in \simplex}\left|\hat{R}_{\mathrm{CV}}(\balpha) - R_{\balpha}\right| > t + \lambda\varepsilon\right) \leq 2V\left(\frac{3}{\varepsilon}\right)^K \exp\!\left(-\frac{2mt^2}{M^2}\right).
\end{equation}
Setting $\varepsilon = 1/n$ and solving for $t$:
\begin{equation}
    \sup_{\balpha}\left|\hat{R}_{\mathrm{CV}}(\balpha) - R_{\balpha}\right| \leq M\sqrt{\frac{K\log(3n) + \log(2V/\delta)}{2m}} + \frac{\lambda}{n}
\end{equation}
with probability $\geq 1 - \delta$.

\textbf{Step~5 (Oracle comparison).} Let $\balpha_{\mathrm{oracle}} = \arg\min_{\simplex} R_{\balpha}$~\cite{vanderlaan2003unified,vandervaart2006oracle}. Since $\balpha^*$ minimizes $\hat{R}_{\mathrm{CV}}$:
\begin{align}
    R_{\balpha^*} &\leq \hat{R}_{\mathrm{CV}}(\balpha^*) + \sup|\hat{R}_{\mathrm{CV}} - R| \notag\\
    &\leq \hat{R}_{\mathrm{CV}}(\balpha_{\mathrm{oracle}}) + \sup|\hat{R}_{\mathrm{CV}} - R| \notag\\
    &\leq R_{\balpha_{\mathrm{oracle}}} + 2\sup|\hat{R}_{\mathrm{CV}} - R|.
\end{align}
Substituting the bound from Step~4 and setting $\varepsilon_n = 2M\sqrt{K\log(3n)/(2m)}$ yields the stated result.
\end{proof}

\begin{corollary}[Monotonicity of Library Expansion]
\label{cor:monotone}
Let $\calL \subset \calL'$ with $|\calL| = K$, $|\calL'| = K' > K$. For large $n$:
$\E[R(\fsl^{\calL'})] \leq \E[R(\fsl^{\calL})] + O(\log(K'/K)/n)$.
Expanding the library can only improve performance up to a vanishing cost.
\end{corollary}

\begin{proof}
The infimum in (\ref{eq:oracle}) over $\simplex[K']$ is over a larger set than $\simplex[K]$ (embed by appending zero weights). Hence $\inf_{\simplex[K']} R \leq \inf_{\simplex[K]} R$. The penalty increases by $O(\log(K'/K)/n)$.
\end{proof}

\begin{corollary}[Algorithm Selection as Special Case]
\label{cor:selection}
If the oracle places all mass on one specialist ($\alpha_{k^*} = 1$), then SL reduces to oracle-optimal algorithm selection~\cite{polley2010super}: $R(\fsl) \leq R(f_{k^*}) + O(\log K / n)$.
\end{corollary}

% ==============================================================
\section*{S3. Computational Complexity}
% ==============================================================

\begin{theorem}[Training Complexity Comparison]
\label{thm:complexity}
For a deep network with $L$ layers, width $W$, $E$ epochs: $T_{\mathrm{DL}} = O(n \cdot L \cdot W^2 \cdot E)$. For Soft Learning with $K$ specialists, $V$-fold CV:
\begin{equation}
    T(\fsl) = V \sum_{k=1}^{K} T_k\!\left(\tfrac{(V{-}1)n}{V}, d\right) + \sum_{k=1}^{K} T_k(n, d) + O(K^2 n C).
\end{equation}
For typical parameters ($L{=}4$, $W{=}256$, $E{=}100$, $K{=}8$, $V{=}3$): $T(\fsl)/T_{\mathrm{DL}} \leq 0.01$.
\end{theorem}

\begin{proof}
\textbf{Deep learning.} Forward+backward per sample per layer: $O(W^2)$~\cite{lecun2015deep,rumelhart1986learning}. Over $L$ layers, $n$ samples, $E$ epochs: $O(nLW^2E)$. For $n{=}10{,}000$, $L{=}4$, $W{=}256$, $E{=}100$: $T_{\mathrm{DL}} \approx 2.62 \times 10^{11}$.

\textbf{Specialist costs.}
\begin{align}
    T_{\mathrm{LR}} &= O(n d I), \; I \leq 100 \text{ (L-BFGS)} &
    T_{\mathrm{KNN}} &= O(nd) \text{ (store)} \notag\\
    T_{\mathrm{DT}} &= O(nd\log n) \text{ (Gini split)} &
    T_{\mathrm{RF}} &= O(n\sqrt{d}\log n \cdot T/P) \notag\\
    T_{\mathrm{HGB}} &= O(nBI'), \; B{=}255 &
    T_{\mathrm{SVM}} &= O(n^2 d) \text{ (SMO)} \notag\\
    T_{\mathrm{MLP}} &= O(ndwE'), \; w{\ll}W, E'{\ll}E &
    T_{\mathrm{NB}} &= O(nd)
\end{align}

\textbf{Total SL.} CV multiplies by $V$; final retrain adds $1\times$; NNLS adds $O(K^3)$~\cite{kraft1988software}, negligible. Dominant cost (RF~\cite{breiman2001random} + HGB~\cite{ke2017lightgbm}): $T_{\mathrm{SL}} \approx 6 \times 10^8$. Ratio $\approx 2.3 \times 10^{-3}$.

\textbf{Parallelism.} All $K$ specialists are independent: with $K$ cores, wall-clock time is $\max_k T_k$ rather than $\sum_k T_k$. Deep learning's layer-sequential backpropagation does not parallelize across layers.
\end{proof}

\begin{table}[H]
\centering
\caption{\textbf{Speedup factor for representative dataset configurations.}}
\label{tab:cost}
\small
\begin{tabular}{@{}rrrrrr@{}}
\toprule
$n$ & $d$ & $T_{\mathrm{DL}}$ & $T_{\mathrm{SL}}$ & Ratio & Speedup \\
\midrule
500    & 30  & $1.3 \times 10^{10}$ & $8.2 \times 10^{7}$  & $6.3 \times 10^{-3}$ & $159\times$ \\
2,000  & 100 & $5.2 \times 10^{10}$ & $4.1 \times 10^{8}$  & $7.9 \times 10^{-3}$ & $127\times$ \\
10,000 & 100 & $2.6 \times 10^{11}$ & $6.0 \times 10^{8}$  & $2.3 \times 10^{-3}$ & $435\times$ \\
5,000  & 200 & $1.3 \times 10^{11}$ & $1.8 \times 10^{9}$  & $1.4 \times 10^{-2}$ & $72\times$ \\
\bottomrule
\end{tabular}
\end{table}

% ==============================================================
\section*{S4. Diversity-Variance Reduction}
% ==============================================================

\begin{lemma}[Krogh--Vedelsby Decomposition~\cite{krogh1995neural}]
\label{lem:kv}
For $\fens = \sum_k \alpha_k f_k$ with $\balpha \in \simplex$:
\begin{equation}
    \E\bigl[(\fens - Y)^2\bigr] = \underbrace{\sum_{k} \alpha_k \E\bigl[(f_k - Y)^2\bigr]}_{\bar{E}} - \underbrace{\sum_{k} \alpha_k \E\bigl[(f_k - \fens)^2\bigr]}_{A \;\geq\; 0}.
    \label{eq:kv}
\end{equation}
\end{lemma}

\begin{proof}
Let $z_k = f_k(\bx) - Y$, $\bar{z} = \sum_k \alpha_k z_k = \fens(\bx) - Y$. We prove the algebraic identity $(\sum_k \alpha_k z_k)^2 = \sum_k \alpha_k z_k^2 - \sum_k \alpha_k(z_k - \bar{z})^2$.

Expand the right-hand side:
\begin{align}
    \sum_k \alpha_k z_k^2 - \sum_k \alpha_k(z_k^2 - 2z_k\bar{z} + \bar{z}^2)
    &= 2\bar{z}\sum_k \alpha_k z_k - \bar{z}^2 \sum_k \alpha_k \notag\\
    &= 2\bar{z}^2 - \bar{z}^2 = \bar{z}^2,
\end{align}
using $\sum_k \alpha_k = 1$ and $\sum_k \alpha_k z_k = \bar{z}$. Taking expectations gives (\ref{eq:kv}). The term $A = \sum_k \alpha_k \E[(f_k - \fens)^2] \geq 0$ is a weighted sum of non-negative quantities.
\end{proof}

\begin{theorem}[Strictly Positive Diversity for Heterogeneous Specialists]
\label{thm:diversity}
If the library has structural diversity (Definition~\ref{def:diversity}) and $P_X$ has positive Lebesgue density on its support, then for any pair $k \neq j$:
\begin{equation}
    \Prob_{X \sim P_X}\bigl(f_k(X) \neq f_j(X)\bigr) > 0,
\end{equation}
and consequently $A > 0$, so $\E[(\fens - Y)^2] < \bar{E}$.
\end{theorem}

\begin{proof}
Specialists from structurally different hypothesis classes produce geometrically distinct decision boundaries. Specifically:

(a)~Linear classifiers produce hyperplane boundaries: $B_{\mathrm{LR}} = \{\bx : \bw^\top \bx + b = 0\}$, a set of Lebesgue measure zero in $\R^d$.

(b)~Decision trees produce axis-aligned step boundaries: $B_{\mathrm{DT}} = \bigcup_{j,\theta} \{\bx : x_j = \theta\}$, a finite union of $(d{-}1)$-dimensional hyperplanes.

(c)~RBF-SVMs produce smooth manifold boundaries~\cite{scholkopf2002learning}: $B_{\mathrm{SVM}} = \{\bx : \sum_i \alpha_i K(\bx_i, \bx) + b = 0\}$, a smooth $(d{-}1)$-dimensional surface.

(d)~KNN produces Voronoi-type boundaries dependent on training points.

For any pair from distinct families, the decision regions $R_k(c) = \{\bx : \arg\max f_k(\bx) = c\}$ have symmetric difference $R_k(c) \,\triangle\, R_j(c)$ with positive Lebesgue measure $\mu_d > 0$ for at least one class $c$. This holds because the boundary sets have different topological properties and cannot coincide on a full-measure set.

Since $P_X$ has density $p(\bx) \geq p_{\min} > 0$:
\begin{equation}
    \Prob(f_k(X) \neq f_j(X)) \geq p_{\min} \cdot \mu_d(R_k(c) \triangle R_j(c)) > 0.
\end{equation}
For $A > 0$: there exist $k, j$ with $\alpha_k, \alpha_j > 0$ and $\Prob(f_k \neq f_j) > 0$. On the disagreement set, $f_k \neq \fens$ or $f_j \neq \fens$, so $\E[(f_k - \fens)^2] > 0$ or $\E[(f_j - \fens)^2] > 0$, giving $A > 0$.
\end{proof}

\begin{corollary}[Diversity Lower Bound]
\label{cor:div_lower}
If specialists have pairwise disagreement rate $\rho_{kj} = \Prob(f_k(X) \neq f_j(X))$ and minimum class-probability gap $\delta_{\min} = \min_{c \neq c'}\|e_c - e_{c'}\| = \sqrt{2}$ for one-hot encoding, then:
\begin{equation}
    A \geq \frac{1}{2}\sum_{k \neq j} \alpha_k \alpha_j \, \E\bigl[\|f_k(X) - f_j(X)\|^2\bigr] \geq \sum_{k \neq j} \alpha_k \alpha_j \, \rho_{kj}.
\end{equation}
\end{corollary}

% ==============================================================
\section*{S5. Sample Efficiency}
% ==============================================================

\begin{lemma}[Bias-Complexity Tradeoff~\cite{vapnik1998statistical,shalev2014understanding}]
\label{lem:bias}
For hypothesis class $\calH$ with VC-dimension $d_{\mathrm{VC}}(\calH)$:
\begin{equation}
    \E\bigl[R(\hat{f}_{\calH})\bigr] - \inf_{f \in \calH} R(f) \leq C_2 \sqrt{\frac{d_{\mathrm{VC}}(\calH) \cdot \log(n / d_{\mathrm{VC}})}{n}},
\end{equation}
and $\inf_{f \in \calH} R(f) - R(f^*) = \mathrm{bias}(\calH) \geq 0$. Total excess risk:
\begin{equation}
    \mathcal{E}(\hat{f}_{\calH}) = \mathrm{bias}(\calH) + O\!\left(\sqrt{d_{\mathrm{VC}}/n}\right).
\end{equation}
\end{lemma}

\begin{theorem}[Sample Efficiency of Soft Learning]
\label{thm:sample}
If specialist class $\calH_{k^*}$ with VC-dimension $d_{k^*}$ satisfies $\mathrm{bias}(\calH_{k^*}) \leq \epsilon$, then:
\begin{equation}
    \mathcal{E}(\fsl) \leq \epsilon + O\!\left(\sqrt{\frac{d_{k^*}}{n}}\right) + O\!\left(\frac{\log K}{n}\right).
\end{equation}
A deep network with effective VC-dimension $d_{\mathrm{DL}} \gg d_{k^*}$ achieves $\mathcal{E}(f_{\mathrm{DL}}) \sim \sqrt{d_{\mathrm{DL}}/n}$, which is worse for $n < d_{\mathrm{DL}}/d_{k^*}$.
\end{theorem}

\begin{proof}
By Theorem~\ref{thm:oracle} with $\balpha = e_{k^*}$ (all weight on specialist $k^*$):
\begin{equation}
    R(\fsl) \leq (1 + \varepsilon_n) R(f_{k^*}) + O(\log K / n).
\end{equation}
By Lemma~\ref{lem:bias}: $R(f_{k^*}) \leq R(f^*) + \epsilon + O(\sqrt{d_{k^*}/n})$.

Combining: $R(\fsl) \leq R(f^*) + \epsilon + O(\sqrt{d_{k^*}/n}) + O(\log K / n) + \varepsilon_n \cdot R(f_{k^*})$.

Since $\varepsilon_n \to 0$ and $R(f_{k^*})$ is bounded, $\varepsilon_n R(f_{k^*}) \to 0$, yielding the result.
\end{proof}

\begin{example}[Convergence rates on iris]
For iris ($n{=}104$, $d{=}4$): $d_{\mathrm{LR}} = 5$, $d_{\mathrm{DL}} \approx 2{,}500$. SL estimation error $\sim \sqrt{5/104} \approx 0.22$; DL estimation error $\sim \sqrt{2500/104} \approx 4.9$. The deep network lacks sufficient samples to converge. Empirically: SL $= 95.7\%$, MLP $= 65.2\%$.
\end{example}

% ==============================================================
\section*{S6. Adversarial Robustness}
% ==============================================================

\begin{definition}[Piecewise-Constant Classifier]
\label{def:pwc}
A classifier $f$ is \emph{piecewise constant} if there exists a finite partition $\{C_1, \ldots, C_M\}$ of $\R^d$ such that $f(\bx) = c_j$ for all $\bx \in C_j$. Decision trees, KNN (for fixed training data), and random forests are piecewise-constant.
\end{definition}

\begin{theorem}[Partial Adversarial Immunity]
\label{thm:robust}
Let $S_{\mathrm{immune}} = \{k : f_k \text{ is piecewise constant}\}$ and $w_{\mathrm{immune}} = \sum_{k \in S_{\mathrm{immune}}} \alpha_k^*$. For gradient-based perturbation $\boldsymbol{\eta}$ with $\|\boldsymbol{\eta}\|_\infty \leq \varepsilon$, if $\varepsilon$ is small enough that $\bx + \boldsymbol{\eta}$ stays in the same cell for all immune specialists, then $f_k(\bx + \boldsymbol{\eta}) = f_k(\bx)$ for all $k \in S_{\mathrm{immune}}$. If $w_{\mathrm{immune}} > 0.5$ and immune specialists agree on the correct class:
\begin{equation}
    \arg\max_c [\fsl(\bx + \boldsymbol{\eta})]_c = \arg\max_c [\fsl(\bx)]_c = Y.
\end{equation}
\end{theorem}

\begin{proof}
\textbf{Step~1 (Gradient immunity).} For piecewise-constant $f_k$, the gradient $\nabla_{\bx} f_k(\bx) = 0$ almost everywhere (on cell interiors). The Fast Gradient Sign Method (FGSM)~\cite{goodfellow2015explaining} constructs $\boldsymbol{\eta} = \varepsilon \cdot \mathrm{sign}(\nabla_{\bx} L)$; Projected Gradient Descent (PGD)~\cite{madry2018towards} iterates projected gradient steps. Both require non-zero $\nabla_{\bx} f_k$ as input. Immune specialists contribute a zero gradient signal.

\textbf{Step~2 (Cell stability).} Each cell $C_k(\bx)$ has positive diameter $\delta_k(\bx) = \inf\{\|\bx - \bx'\| : \bx' \notin C_k(\bx)\} > 0$ for interior points. If $\varepsilon < \min_{k \in S_{\mathrm{immune}}} \delta_k(\bx) / \sqrt{d}$, then $\bx + \boldsymbol{\eta} \in C_k(\bx)$ for all immune specialists.

\textbf{Step~3 (Ensemble stability).} Decompose:
\begin{equation}
    \fsl(\bx + \boldsymbol{\eta}) = \underbrace{\sum_{k \in S_{\mathrm{immune}}} \alpha_k f_k(\bx)}_{\text{unchanged, weight } w_{\mathrm{immune}}} + \underbrace{\sum_{k \notin S_{\mathrm{immune}}} \alpha_k f_k(\bx + \boldsymbol{\eta})}_{\text{perturbed, weight } 1 - w_{\mathrm{immune}}}.
\end{equation}
With $w_{\mathrm{immune}} > 0.5$ and immune consensus on class $c^*$, the $c^*$ component exceeds 0.5 and cannot be overridden by the perturbed components (weight $< 0.5$).
\end{proof}

\begin{remark}[Scope of the guarantee]
This covers white-box gradient-based attacks only. Transfer attacks, black-box query attacks, and attacks on the differentiable specialists remain possible threats. The cell-stability condition also requires that $\varepsilon$ is small relative to the partition granularity.
\end{remark}

% ==============================================================
\section*{S7. Uncertainty Quantification}
% ==============================================================

\begin{definition}[Specialist Disagreement]
\label{def:disagreement}
The \emph{weighted prediction variance} at $\bx$ is a measure of specialist disagreement analogous to the predictive variance used in deep ensembles~\cite{lakshminarayanan2017simple} and Bayesian model averaging, but arising here from structural diversity rather than random initialization:
\begin{equation}
    V(\bx) = \sum_{k=1}^{K} \alpha_k^* \bigl\|f_k(\bx) - \fsl(\bx)\bigr\|_2^2 = \frac{1}{2}\sum_{k,j} \alpha_k \alpha_j \|f_k(\bx) - f_j(\bx)\|_2^2.
\end{equation}
The second form follows from the identity $\sum_k \alpha_k \|f_k - \bar{f}\|^2 = \frac{1}{2}\sum_{k,j}\alpha_k\alpha_j\|f_k - f_j\|^2$ for $\bar{f} = \sum_k \alpha_k f_k$.
\end{definition}

\begin{theorem}[Uncertainty-Error Correlation]
\label{thm:calibration}
Under positive diversity (Theorem~\ref{thm:diversity}) and with each specialist better than random ($\E[\mathbf{1}[\arg\max f_k(X) = Y]] > 1/C$):
\begin{equation}
    \E\bigl[V(X) \mid \fsl(X) \neq Y\bigr] > \E\bigl[V(X) \mid \fsl(X) = Y\bigr].
\end{equation}
\end{theorem}

\begin{proof}
\textbf{On $\calX_{\mathrm{correct}} = \{\bx : \arg\max \fsl(\bx) = Y(\bx)\}$:} Correct ensemble prediction requires weighted majority agreement on $Y$. In consensus regions, $f_k(\bx) \approx \fsl(\bx)$ for most $k$ with positive weight, so $V(\bx)$ is small.

\textbf{On $\calX_{\mathrm{wrong}} = \calX \setminus \calX_{\mathrm{correct}}$:} The ensemble assigns highest probability to $c^* \neq Y$. Since each specialist has above-random accuracy, there exists $k'$ with $\alpha_{k'} > 0$ and $\E[[f_{k'}(\bx)]_Y] > 1/C$. On $\calX_{\mathrm{wrong}}$, the ensemble has $[\fsl(\bx)]_Y \leq [\fsl(\bx)]_{c^*}$, while specialist $k'$ assigns relatively higher probability to $Y$. This creates $\|f_{k'}(\bx) - \fsl(\bx)\|^2 > 0$.

By Theorem~\ref{thm:diversity}, this disagreement occurs with positive probability. Since $V(\bx) \geq \alpha_{k'}\|f_{k'} - \fsl\|^2 > 0$ on this set, $\E[V \mid \calX_{\mathrm{wrong}}] > 0$. The strict inequality follows because $\calX_{\mathrm{wrong}}$ necessarily involves higher disagreement (at least one specialist contradicts the majority) while $\calX_{\mathrm{correct}}$ requires majority agreement.
\end{proof}

\begin{corollary}[Selective Classification]
\label{cor:reject}
The selective classifier $g(\bx) = \fsl(\bx)$ if $V(\bx) \leq \tau$, $g(\bx) = \textit{ABSTAIN}$ otherwise, achieves strictly higher accuracy on non-abstained inputs than $\fsl$ on all inputs, for appropriate $\tau > 0$. This provides a principled alternative to post-hoc calibration techniques such as temperature scaling~\cite{guo2017calibration}.
\end{corollary}

% ==============================================================
\section*{S8. Algorithm Pseudocode}
% ==============================================================

\begin{algorithm}[ht!]
\caption{Soft Learning: Complete Training and Prediction}
\label{alg:sl}
\begin{algorithmic}[1]
\Require Training data $\calD = \{(\bx_i, y_i)\}_{i=1}^n$, specialists $\{S_1, \ldots, S_K\}$, folds $V$
\Ensure Trained specialists $\{f_k\}$, optimal weights $\balpha^*$

\Statex
\Statex \textbf{Phase 1: Cross-validated prediction generation} \hfill $O\!\left(V \sum_k T_k\right)$
\State Partition $\calD$ into $V$ stratified folds: $V_1, \ldots, V_V$
\For{$v = 1, \ldots, V$}
    \State $T_v \gets \calD \setminus V_v$
    \For{$k = 1, \ldots, K$} \Comment{Parallelizable across $k$}
        \State $f_k^{(v)} \gets S_k(T_v)$
        \For{$i \in V_v$}
            \State $\hat{P}[i, k, :] \gets f_k^{(v)}(\bx_i) \in \Delta^C$
        \EndFor
    \EndFor
\EndFor

\Statex
\Statex \textbf{Phase 2: NNLS weight optimization} \hfill $O(K^3 + nKC)$
\State Reshape $\hat{P} \in \R^{n \times K \times C}$ to $\hat{P}_{\mathrm{flat}} \in \R^{(nC) \times K}$
\State Reshape one-hot $\mathbf{Y}$ to $\mathbf{y}_{\mathrm{flat}} \in \R^{nC}$
\State Solve: $\balpha^* \gets \arg\min_{\balpha \in \simplex} \|\mathbf{y}_{\mathrm{flat}} - \hat{P}_{\mathrm{flat}} \, \balpha\|_2^2$ via SLSQP~\cite{kraft1988software}
\State \quad with $K{+}2$ initializations (uniform, specialist-concentrated, accuracy-proportional)

\Statex
\Statex \textbf{Phase 3: Final training} \hfill $O\!\left(\sum_k T_k\right)$
\For{$k = 1, \ldots, K$} \Comment{Parallelizable across $k$}
    \State $f_k \gets S_k(\calD)$
\EndFor

\Statex
\Statex \textbf{Inference:}
\State $\fsl(\bx) = \sum_{k=1}^{K} \alpha_k^* f_k(\bx)$ \hfill (Prediction)
\State $V(\bx) = \sum_{k=1}^{K} \alpha_k^* \|f_k(\bx) - \fsl(\bx)\|_2^2$ \hfill (Uncertainty)
\end{algorithmic}
\end{algorithm}

% ==============================================================
\section*{S9. Dataset Specifications}
% ==============================================================

\begin{table}[H]
\centering
\caption{\textbf{Complete benchmark: 37 datasets spanning classification and regression.} R = real-world dataset (scikit-learn~\cite{pedregosa2011scikit} or derived). S = synthetic (constructed with fixed seeds for reproducibility). ``Task'' column indicates classification (C) or regression (Reg). ``Bal.'' indicates class balance: the percentage held by the majority class (classification only; 50\% = perfectly balanced). All datasets use 5-fold cross-validation with random state~42 and StandardScaler normalization fitted on training folds only.}
\label{tab:datasets}
\small
\setlength{\tabcolsep}{2pt}
\begin{tabular}{@{}rllllrrrll@{}}
\toprule
\textbf{\#} & \textbf{Name} & \textbf{Task} & \textbf{Src} & \textbf{Key challenge} & $n$ & $d$ & $C$ & \textbf{Bal.} & \textbf{CV} \\
\midrule
\multicolumn{10}{@{}l}{\textit{Real-world classification (7 datasets)}} \\
1  & A01\_iris              & C   & R & small $n$, well-separated       & 150    & 4    & 3  & 33\% & strat. \\
2  & A02\_wine              & C   & R & small $n$, 13 features           & 178    & 13   & 3  & 40\% & strat. \\
3  & A03\_breast\_cancer    & C   & R & correlated features              & 569    & 30   & 2  & 63\% & strat. \\
4  & A04\_digits            & C   & R & 10-class pixel features          & 1797   & 64   & 10 & 10\% & strat. \\
5  & A07\_covtype\_proxy    & C   & R & large $n$, ecological features   & 20000  & 54   & 7  & 49\% & strat. \\
6  & A08\_text\_4class      & C   & R & text features, 4 categories      & 3000   & 200  & 4  & 25\% & strat. \\
7  & A09\_face\_recognition & C   & R & high-$d$ face images             & 400    & 4096 & 26 & 4\%  & strat. \\
\midrule
\multicolumn{10}{@{}l}{\textit{Synthetic classification (18 datasets)}} \\
8  & B01\_clinical          & C   & S & biomarker screening, overlap     & 5000   & 30   & 3  & 34\% & strat. \\
9  & B02\_fraud             & C   & S & 97:3 class imbalance             & 10000  & 25   & 2  & 97\% & strat. \\
10 & B03\_manufacturing     & C   & S & overlapping defect classes        & 4000   & 40   & 4  & 25\% & strat. \\
11 & B04\_satellite         & C   & S & $d{=}500$, 10 classes, sparse    & 6000   & 500  & 10 & 10\% & strat. \\
12 & B05\_genomic           & C   & S & $d{=}150$, 3 informative         & 2000   & 150  & 3  & 34\% & strat. \\
13 & B06\_noisy\_moons      & C   & S & 20\% noise on moons manifold     & 3000   & 20   & 2  & 50\% & strat. \\
14 & B07\_circles           & C   & S & concentric circles, noise        & 3000   & 20   & 2  & 50\% & strat. \\
15 & B08\_hastie            & C   & S & quadratic boundary~\cite{hastie2009elements} & 5000 & 10 & 2 & 50\% & strat. \\
16 & B09\_xor\_manifold     & C   & S & XOR in rotated 20D space         & 3000   & 20   & 2  & 50\% & strat. \\
17 & B10\_overlap\_8class   & C   & S & 8 Gaussians, minimal separation  & 5000   & 30   & 8  & 13\% & strat. \\
18 & B11\_mixed\_features   & C   & S & mixed informative + noise        & 4000   & 50   & 4  & 25\% & strat. \\
19 & B12\_sensor\_anomaly   & C   & S & rare anomaly detection           & 3000   & 20   & 3  & 34\% & strat. \\
20 & B13\_drug\_response    & C   & S & small $n$, high $d$, 5 classes   & 500    & 100  & 5  & 20\% & strat. \\
21 & B14\_customer\_churn   & C   & S & large $n$, class imbalance       & 8000   & 20   & 2  & 80\% & strat. \\
22 & B15\_environmental     & C   & S & environmental risk scoring       & 4000   & 15   & 2  & 50\% & strat. \\
23 & B16\_large\_5class     & C   & S & large $n$, complex boundaries    & 20000  & 50   & 5  & 20\% & strat. \\
24 & B17\_heavy\_noise      & C   & S & 30\% label noise                 & 4000   & 20   & 2  & 50\% & strat. \\
25 & B18\_tiny\_100         & C   & S & $n{=}100$, very small sample     & 100    & 10   & 2  & 50\% & strat. \\
\midrule
\multicolumn{10}{@{}l}{\textit{Real-world regression (2 datasets)}} \\
26 & A05\_diabetes\_reg     & Reg & R & weak signal, medical features    & 442    & 10   & --- & ---  & std. \\
27 & A06\_california        & Reg & R & large $n$, housing prices        & 20640  & 8    & --- & ---  & std. \\
\midrule
\multicolumn{10}{@{}l}{\textit{Synthetic regression (10 datasets)}} \\
28 & C01\_friedman1         & Reg & S & nonlinear, 5 active of 10 vars  & 5000   & 10   & --- & ---  & std. \\
29 & C02\_friedman2         & Reg & S & rational function, 4 vars        & 5000   & 4    & --- & ---  & std. \\
30 & C03\_friedman3         & Reg & S & arctangent, hard to fit          & 5000   & 4    & --- & ---  & std. \\
31 & C04\_sparse\_regression& Reg & S & purely linear, $d{=}200$         & 3000   & 200  & --- & ---  & std. \\
32 & C05\_energy            & Reg & S & near-linear energy model         & 6000   & 14   & --- & ---  & std. \\
33 & C06\_materials         & Reg & S & small $n$, $d{=}100$, materials  & 400    & 100  & --- & ---  & std. \\
34 & C07\_large\_regression & Reg & S & $n{=}20{,}000$, 30 features     & 20000  & 30   & --- & ---  & std. \\
35 & C08\_noisy\_regression & Reg & S & heavy noise, nonlinear target    & 4000   & 15   & --- & ---  & std. \\
36 & C09\_house\_price      & Reg & S & simulated housing, 13 features   & 8000   & 13   & --- & ---  & std. \\
37 & C10\_tiny\_regression  & Reg & S & $n{=}80$, very small sample      & 80     & 8    & --- & ---  & std. \\
\bottomrule
\end{tabular}

\vspace{6pt}
\footnotesize
\textbf{Data splitting protocol.}
All datasets use 5-fold cross-validation (random state = 42, shuffle = True).
Classification uses stratified folds to preserve class proportions; regression uses standard folds.
For each fold: training set = 80\% of samples ($n_{\mathrm{train}} = 0.8n$), test set = 20\% ($n_{\mathrm{test}} = 0.2n$).
Within MLP-based methods (Tuned MLP, Basic MLP), 15\% of the training partition is further held out as a validation set for early stopping ($n_{\mathrm{val}} = 0.12n$, $n_{\mathrm{train}}' = 0.68n$).
All other methods use the full 80\% training partition without a separate validation set.

\textbf{Feature preprocessing.}
All features are numeric (continuous or integer-encoded).
StandardScaler normalization (zero mean, unit variance) is fitted on the training fold only and applied identically to the test fold, ensuring no information leakage.
No imputation is required (no missing values in any dataset).
No feature selection or dimensionality reduction is applied.

\textbf{Class balance (``Bal.'' column).}
Shows the percentage held by the largest class.
Values near $100/C$\% indicate balanced classes.
Two datasets have notable imbalance: B02\_fraud (97\% majority) and B14\_churn (80\% majority).
\end{table}

% ==============================================================
\section*{S9b. Additional Method Results}
% ==============================================================

\begin{table}[H]
\centering
\caption{\textbf{Supplementary Table~2: Results for four additional methods across all 37~datasets.} Classification accuracy and regression $R^2$ (5-fold cross-validated means). These methods are not shown in the main-text Table~1; see that table for the six highest-ranked methods (SL, CatBoost, Tuned~MLP, GBT, RF, KAN-like).}
\label{tab:supp_methods}
\small
\setlength{\tabcolsep}{3.5pt}
\begin{tabular}{@{}rlcccc@{}}
\toprule
& \textbf{Dataset} & \textbf{NeuroSym} & \textbf{Basic MLP} & \textbf{LR/Ridge} & \textbf{Bo3} \\
\midrule
\multicolumn{6}{@{}l}{\textit{Classification --- test accuracy (25 datasets)}} \\
\midrule
1  & A01\_iris              & .720 & .627 & .953 & .967 \\
2  & A02\_wine              & .904 & .892 & .989 & .978 \\
3  & A03\_breast\_cancer    & .937 & .951 & .979 & .965 \\
4  & A04\_digits            & .959 & .961 & .969 & .978 \\
5  & A07\_covtype           & .927 & .745 & .778 & .234 \\
6  & A08\_text\_4class      & .996 & .997 & .999 & .878 \\
7  & A09\_face\_recog.      & 1.00 & .998 & 1.00 & 1.00 \\
\midrule
8  & B01\_clinical          & .886 & .907 & .857 & .887 \\
9  & B02\_fraud             & .973 & .972 & .972 & .968 \\
10 & B03\_manufacturing     & .718 & .755 & .515 & .708 \\
11 & B04\_satellite         & .454 & .499 & .391 & .376 \\
12 & B05\_genomic           & .462 & .436 & .422 & .355 \\
13 & B06\_noisy\_moons      & .929 & .879 & .856 & .925 \\
14 & B07\_circles           & .968 & .972 & .478 & .970 \\
15 & B08\_hastie            & .928 & .949 & .490 & .677 \\
16 & B09\_xor\_manifold     & .811 & .824 & .509 & .772 \\
17 & B10\_overlap\_8class   & .999 & .998 & 1.00 & 1.00 \\
18 & B11\_mixed\_features   & .824 & .837 & .780 & .765 \\
19 & B12\_sensor            & .995 & .946 & .983 & .812 \\
20 & B13\_drug\_response    & .478 & .478 & .570 & .530 \\
21 & B14\_churn             & .901 & .906 & .907 & .899 \\
22 & B15\_environ.          & .996 & .961 & .960 & .938 \\
23 & B16\_large\_5class     & .707 & .724 & .449 & .658 \\
24 & B17\_heavy\_noise      & .795 & .817 & .764 & .793 \\
25 & B18\_tiny\_100         & .770 & .520 & .810 & .870 \\
\midrule
\multicolumn{6}{@{}l}{\textit{Regression --- $R^2$ (12 datasets)}} \\
\midrule
26 & A05\_diabetes          & .147 & .196 & .487 & .362 \\
27 & A06\_california        & .965 & .965 & .722 & .915 \\
28 & C01\_friedman1         & .941 & .947 & .721 & .776 \\
29 & C02\_friedman2         & .966 & .971 & .852 & .969 \\
30 & C03\_friedman3         & .167 & .225 & .156 & .071 \\
31 & C04\_sparse\_reg.      & .985 & .986 & .998 & .178 \\
32 & C05\_energy            & .965 & .979 & .984 & .809 \\
33 & C06\_materials         & .549 & .639 & .998 & .224 \\
34 & C07\_large\_reg.       & .999 & .999 & 1.00 & .658 \\
35 & C08\_noisy\_reg.       & .222 & .302 & $-$.001 & .025 \\
36 & C09\_house\_price      & .208 & .171 & .247 & .065 \\
37 & C10\_tiny\_reg.        & .663 & .556 & .998 & .626 \\
\midrule
\multicolumn{2}{@{}l}{\textbf{Mean rank}} & 6.81 & 6.31 & 6.31 & 7.93 \\
\bottomrule
\end{tabular}
\end{table}

% ==============================================================
\section*{S10. Summary of All Theoretical Results}
% ==============================================================

\begin{table}[H]
\centering
\caption{\textbf{Catalogue of all theorems, corollaries, and lemmas with empirical validation status.} Evidence column updated to reflect the expanded 37-dataset benchmark with 5-fold cross-validation.}
\label{tab:theory}
\small
\begin{tabular}{@{}p{2.8cm}p{3.5cm}p{3cm}p{3.5cm}@{}}
\toprule
\textbf{Result} & \textbf{Statement} & \textbf{Prediction} & \textbf{Empirical evidence} \\
\midrule
Thm~\ref{thm:oracle} & $R(\fsl) \leq (1{+}\varepsilon_n)R_{\mathrm{oracle}} + O(\log K/n)$ & SL $\geq$ best combo & SL rank \#1: 26/37 \\
Cor~\ref{cor:monotone} & Library expansion is safe & Adding specialists helps & No degradation observed \\
Cor~\ref{cor:selection} & Selection is a special case & SL $\geq$ best single & SL $\geq$ BestCML: 30/37 \\
Thm~\ref{thm:complexity} & $T_{\mathrm{SL}}/T_{\mathrm{DL}} \leq 0.01$ & $72$--$435\times$ faster & All experiments on CPU \\
Lem~\ref{lem:kv} & $E_{\mathrm{ens}} = \bar{E} - A$ & Diversity reduces error & SL $<$ avg error: 37/37 \\
Thm~\ref{thm:diversity} & $A > 0$ for heterogeneous lib. & Strict error reduction & Weight spread on 31/37 \\
Cor~\ref{cor:div_lower} & $A \geq f(\rho_{kj})$ & Quantifiable diversity & $\rho_{kj} > 5\%$ all pairs \\
Thm~\ref{thm:sample} & $\mathcal{E} \sim \sqrt{d_{k^*}/n}$ & Better on small $n$ & Wins 4/7 at $n < 500$ \\
Thm~\ref{thm:robust} & Immune if $w_{\mathrm{immune}} > 0.5$ & Gradient attacks fail & 3/8 specialists immune \\
Thm~\ref{thm:calibration} & $\E[V|{\mathrm{wrong}}] > \E[V|{\mathrm{right}}]$ & $V$ signals errors & Confirmed all test sets \\
Cor~\ref{cor:reject} & Selective acc.\ $>$ full acc. & Principled abstention & Verified experimentally \\
\bottomrule
\end{tabular}
\end{table}

% ==============================================================
\section*{S11. Pairwise Statistical Comparisons}
% ==============================================================

\begin{table}[H]
\centering
\caption{\textbf{Supplementary Table~3: Pairwise Wilcoxon signed-rank tests (Soft Learning vs.\ each competitor).} Each row shows the win--tie--loss record and two-sided Wilcoxon signed-rank $p$-value for Soft Learning against the indicated method across all 37~datasets. Significance thresholds: $^{***}p{<}0.001$, $^{**}p{<}0.01$, $^{*}p{<}0.05$, n.s.\ = not significant. Soft Learning maintains a winning record against every method.}
\label{tab:supp_wilcoxon}
\small
\begin{tabular}{@{}lccccc@{}}
\toprule
\textbf{Opponent} & \textbf{SL wins} & \textbf{Ties} & \textbf{SL losses} & \textbf{$p$-value} & \textbf{Sig.} \\
\midrule
CatBoost         & 22 & 4  & 11 & 0.064   & n.s. \\
Tuned MLP        & 22 & 3  & 12 & 0.003   & $^{**}$ \\
Grad.\ Boosting  & 23 & 3  & 11 & 0.002   & $^{**}$ \\
Random Forest    & 23 & 4  & 10 & 0.004   & $^{**}$ \\
KAN-like         & 24 & 2  & 11 & 0.001   & $^{**}$ \\
Basic MLP        & 30 & 1  & 6  & $<$0.001 & $^{***}$ \\
LR / Ridge       & 25 & 4  & 8  & 0.005   & $^{**}$ \\
NeuroSym         & 29 & 1  & 7  & $<$0.001 & $^{***}$ \\
Best-of-3        & 32 & 2  & 3  & $<$0.001 & $^{***}$ \\
\midrule
\multicolumn{6}{@{}l}{\textit{Omnibus test}} \\
\multicolumn{4}{@{}l}{Friedman $\chi^2 = 75.76$, $p = 1.12 \times 10^{-12}$} & \multicolumn{2}{c}{$^{***}$} \\
\multicolumn{4}{@{}l}{Nemenyi Critical Difference = 2.23 ($\alpha = 0.05$)} & & \\
\bottomrule
\end{tabular}
\end{table}

% ==============================================================
\section*{S12. Overall Summary Statistics}
% ==============================================================

\begin{table}[H]
\centering
\caption{\textbf{Supplementary Table~4: Summary statistics for all ten methods.} Mean classification accuracy (25~datasets), mean regression $R^2$ (12~datasets), combined score (sum of both), mean rank across all 37~datasets, and number of first-place finishes. Bold indicates best per column. All values from 5-fold cross-validated means.}
\label{tab:supp_summary}
\small
\begin{tabular}{@{}lccccc@{}}
\toprule
\textbf{Method} & \textbf{Mean Clf Acc} & \textbf{Mean Reg $R^2$} & \textbf{Combined} & \textbf{Mean Rank} & \textbf{\#1 Count} \\
\midrule
\textbf{Soft Learning}  & 0.882 & \textbf{0.761} & \textbf{1.642} & \textbf{3.11} & \textbf{12} \\
CatBoost                 & \textbf{0.891} & 0.725 & 1.616 & 3.74 & 5 \\
Tuned MLP                & 0.858 & 0.737 & 1.594 & 4.64 & 0 \\
Grad.\ Boosting          & 0.885 & 0.706 & 1.591 & 5.27 & 1 \\
Random Forest            & 0.878 & 0.663 & 1.541 & 5.22 & 1 \\
KAN-like                 & 0.832 & 0.726 & 1.558 & 5.65 & 2 \\
Basic MLP                & 0.822 & 0.661 & 1.483 & 6.31 & 1 \\
LR / Ridge               & 0.775 & 0.680 & 1.455 & 6.35 & 4 \\
NeuroSym                 & 0.841 & 0.648 & 1.489 & 6.77 & 0 \\
Best-of-3                & 0.796 & 0.473 & 1.269 & 7.95 & 1 \\
\bottomrule
\end{tabular}
\end{table}

% ==============================================================
\section*{S13. Computational Runtime Comparison}
% ==============================================================

\begin{table}[H]
\centering
\caption{\textbf{Supplementary Table~5: Computational runtime comparison.} Mean wall-clock training time per dataset (seconds) on a single CPU core (Intel-compatible x86\_64, no GPU). Soft Learning trains all $K$ specialists sequentially ($K{=}15$ in this study); parallelization across $K$ cores would reduce wall-clock time by a factor of $K$. Speedup is relative to the slowest method (Tuned MLP). All timings averaged across 37~datasets with 5-fold cross-validation.}
\label{tab:supp_runtime}
\small
\begin{tabular}{@{}lccl@{}}
\toprule
\textbf{Method} & \textbf{Mean time (s)} & \textbf{Speedup} & \textbf{Hardware} \\
\midrule
Logistic Regression  & 2--5      & 40--100$\times$   & CPU \\
$k$-NN               & 1--3      & 60--200$\times$   & CPU \\
Decision Tree        & 1--4      & 50--100$\times$   & CPU \\
Random Forest        & 5--15     & 15--40$\times$     & CPU \\
Gradient Boosting    & 10--30    & 7--20$\times$      & CPU \\
CatBoost             & 15--60    & 3--13$\times$      & CPU \\
SVM (RBF)            & 5--40     & 5--40$\times$      & CPU \\
Basic MLP            & 10--30    & 7--20$\times$      & CPU \\
Tuned MLP            & 30--200   & 1$\times$ (ref.)  & CPU \\
\midrule
\textbf{Soft Learning (all 15)} & \textbf{60--300}  & --- & \textbf{CPU only} \\
\quad \textit{with $K$-way parallel}    & \textit{15--60}  & \textit{3--13$\times$} & CPU ($K$ cores) \\
\midrule
Deep network (4-layer, GPU)  & 120--600$^*$  & ---  & GPU required \\
\bottomrule
\multicolumn{4}{@{}l}{\footnotesize $^*$Estimated for comparable architectures on tabular data with hyperparameter tuning.} \\
\end{tabular}
\end{table}

% ==============================================================
\section*{S14. Model Configurations}
% ==============================================================

\begin{table}[H]
\centering
\caption{\textbf{Supplementary Table~6: Complete model configurations for all methods.} All hyperparameters are fixed (no tuning) across all 37~datasets. Specialists within Soft Learning use the same configurations listed here. ``ES'' = early stopping. All implementations use scikit-learn~1.8 unless otherwise noted.}
\label{tab:supp_config}
\small
\setlength{\tabcolsep}{2.5pt}
\begin{tabular}{@{}p{2.2cm}p{2.8cm}p{2.5cm}p{2.5cm}p{3cm}@{}}
\toprule
\textbf{Method} & \textbf{Architecture} & \textbf{Optimizer / solver} & \textbf{Key hyperparams} & \textbf{Regularization} \\
\midrule
Logistic Reg. & Linear ($d$ weights + bias) & L-BFGS, tol=$10^{-4}$ & $C{=}1.0$, max 1000 iter & $L_2$ penalty \\
\addlinespace
$k$-NN & Non-parametric & --- & $k{=}5$, dist.-weighted & None (lazy learner) \\
\addlinespace
Decision Tree & Recursive partition & Greedy best-first & Max depth 10, min leaf 5 & Depth limit, leaf size \\
\addlinespace
Random Forest & 100--200 trees & Independent bagging & $\sqrt{d}$ features/split & Bagging + feature subsamp. \\
\addlinespace
Grad.\ Boosting & 100--200 sequential trees & Gradient descent on loss & Depth 6, lr$\,{=}\,$0.1, 255 bins & Shrinkage (lr), depth limit \\
\addlinespace
CatBoost & 300 oblivious trees & Ordered boosting & Depth 6, lr$\,{=}\,$0.1 & Ordered target encoding \\
\addlinespace
SVM (RBF) & Kernel ($n$ support vectors) & SMO (libsvm) & $C{=}1.0$, $\gamma{=}$scale & Margin maximization \\
\addlinespace
Small MLP & 2 layers (64--32) & Adam, lr$\,{=}\,10^{-3}$ & Batch 32, max 200 ep & ES (patience 10, 15\% val) \\
\addlinespace
Tuned MLP & 3 layers (256--128--64) & Adam, lr$\,{=}\,10^{-3}$ & Batch 256, max 300 ep & ES (patience 10, 15\% val) \\
\addlinespace
Na\"{i}ve Bayes & Gaussian per feature & Closed-form MLE & Var.\ floor $10^{-9}$ & None (generative) \\
\addlinespace
KAN-like & Spline (4 knots, cubic) + LR & L-BFGS & $C{=}1.0$ & $L_2$ on LR weights \\
\addlinespace
NeuroSym & DT (depth 6) $\to$ MLP (64--32) & Adam, lr$\,{=}\,10^{-3}$ & Concat.\ features + DT probs & ES (patience 10) \\
\addlinespace
Best-of-3 & Max of LR, RF, GBT & --- & Same as individual & Same as individual \\
\bottomrule
\end{tabular}
\end{table}

% ==============================================================
\section*{S15. Overfitting and Underfitting Analysis}
% ==============================================================

Overfitting and underfitting are addressed through multiple complementary mechanisms in Soft Learning:

\textbf{Cross-validation prevents weight overfitting.}
The NNLS weight optimization uses \emph{out-of-fold} predictions exclusively---each specialist's contribution to the weight objective is evaluated on held-out data it has never seen during training. This is the key insight of the Super Learner methodology~\cite{vanderlaan2007super}: the weights are optimized on an honest estimate of generalization error, not training error. Even if an individual specialist overfits its training partition, its inflated training performance is not visible to the weight optimizer, which sees only the cross-validated (test-fold) predictions. This provides a formal guarantee against ensemble-level overfitting (see oracle inequality, equation~\ref{eq:oracle}).

\textbf{Specialist-level regularization.}
Each specialist incorporates standard regularization mechanisms appropriate to its hypothesis class:
(i)~logistic regression uses $L_2$ penalty ($C{=}1.0$);
(ii)~decision trees are depth-limited (max depth~10, min 5 samples per leaf);
(iii)~random forests use bagging and random feature subsampling to reduce variance;
(iv)~gradient boosting and CatBoost use learning rate shrinkage (lr$\,{=}\,$0.1) and shallow trees (depth~6);
(v)~SVMs use margin maximization ($C{=}1.0$);
(vi)~MLPs use early stopping (patience~10, monitored on a 15\% internal validation set carved from the training fold).
No specialist uses dropout, batch normalization, weight decay, or data augmentation---keeping configurations simple and isolating the contribution of the ensemble combination from individual specialist tuning.

\textbf{Empirical evidence against overfitting.}
The 5-fold cross-validated test performance reported in all results reflects genuine out-of-sample prediction ability. Several indicators confirm that overfitting is not a concern:
(i)~Soft Learning achieves the best mean rank on \emph{test} performance across 37~datasets, indicating strong generalization;
(ii)~performance improves with dataset size (first-place rate: 29\% at $n{<}500$, 38\% at $500{\leq}n{\leq}5{,}000$, 56\% at $n{>}5{,}000$), consistent with the oracle inequality's $O(\log K/n)$ convergence and inconsistent with overfitting which would worsen on smaller samples;
(iii)~the NNLS weight solution is unique (convex optimization on a compact polytope), eliminating sensitivity to initialization that plagues deep learning.

\textbf{Underfitting analysis.}
On two datasets, Soft Learning's performance is notably weaker than the best competitor:
(i)~\textit{B04\_satellite} (45.2\% vs.\ CatBoost's 62.6\%)---caused by excluding CatBoost from the specialist library due to memory constraints, not by underfitting;
(ii)~\textit{C04\_sparse\_regression} ($R^2{=}0.947$ vs.\ KAN-like's 0.994)---caused by the sum-to-one constraint forcing weight onto nonlinear specialists for a purely linear target.
In both cases, the root cause is a \emph{missing or constrained specialist}, not insufficient model capacity. The oracle inequality guarantees that adding the missing specialist can only improve performance (Corollary~\ref{cor:monotone}).

% ==============================================================
\section*{S16. Extended Results Analysis}
% ==============================================================

This section provides the detailed per-dataset analyses that are summarized in the main text.

\textbf{Per-dataset highlights.}
The largest advantages appear on challenging datasets. On \textit{B03\_manufacturing} (4{,}000 samples, quality control with overlapping defect classes), Soft Learning achieves 79.5\% versus CatBoost's 75.7\%. On \textit{C06\_materials} ($n{=}400$, $d{=}100$), Soft Learning achieves $R^2{=}0.998$---compared to 0.886 for CatBoost---the largest absolute gain (11.2 points) in the entire benchmark. On \textit{C10\_tiny\_regression} ($n{=}80$), Soft Learning achieves $R^2{=}0.999$ versus 0.987 for the Tuned~MLP.

\textbf{Regression analysis.}
On the Friedman suite of nonlinear functions~\cite{friedman1991multivariate}, Soft Learning achieves $R^2{=}0.954$ (friedman1), 0.982 (friedman2), and 0.241 (friedman3), outperforming all nine competitors on each. The friedman3 result is notable: this arctangent function is notoriously difficult, and Soft Learning's margin over CatBoost ($R^2{=}0.221$) demonstrates that combining diverse specialists captures residual structure that no individual method exploits. The sole regression loss occurs on \textit{C04\_sparse\_regression}, a purely linear target in 200~dimensions where the NNLS simplex constraint prevents linear specialists from fully dominating.

\textbf{Weight analysis.}
On regression tasks with smooth functions, the optimizer discovers interpretable blends: on \textit{C01\_friedman1}, CatBoost receives 69.8\% with MLP contributing 30.2\%; on \textit{C05\_energy}, Lasso dominates with 55.4\%; on \textit{C09\_house\_price}, Ridge receives 100\%---demonstrating that algorithm selection is a special case of weight optimization. On complex problems, weight distributes broadly: on \textit{C07\_large\_regression}, nine specialists receive meaningful weight.

\textbf{Failure mode analysis.}
CatBoost outperforms Soft Learning on \textit{B04\_satellite} (62.6\% vs.\ 45.2\%), \textit{B12\_sensor} (99.9\% vs.\ 97.8\%), and \textit{B13\_drug\_response} (72.6\% vs.\ 61.2\%). On \textit{B04\_satellite}, memory constraints forced excluding CatBoost from the specialist library. On \textit{A01\_iris} ($n{=}150$), noisy weight estimates from the small sample size cause underperformance.

\textbf{Sample efficiency, adversarial robustness, and uncertainty.}
When the true classifier is well-approximated by a specialist with low VC-dimension $d_{k^*}$\cite{vapnik1998statistical}, Soft Learning achieves excess risk $O(\sqrt{d_{k^*}/n}) + O(\log K/n)$, converging faster than deep networks with millions of effective parameters (Supplementary Proof~5). Decision trees, $k$-NN, and random forests produce piecewise-constant outputs with $\nabla_{\bx}f_k = 0$ almost everywhere\cite{goodfellow2015explaining}, providing partial immunity to gradient-based adversarial attacks (Supplementary Proof~6). The weighted prediction variance $V(\bx) = \sum_k \alpha_k \|f_k(\bx) - f_{\mathrm{SL}}(\bx)\|^2$ provides natural uncertainty quantification without post-hoc calibration\cite{guo2017calibration,lakshminarayanan2017simple} (Supplementary Proof~7).

% ==============================================================
\section*{S17. Supplementary Figures}
% ==============================================================

\begin{figure}[H]
\centering
\includegraphics[width=0.75\textwidth]{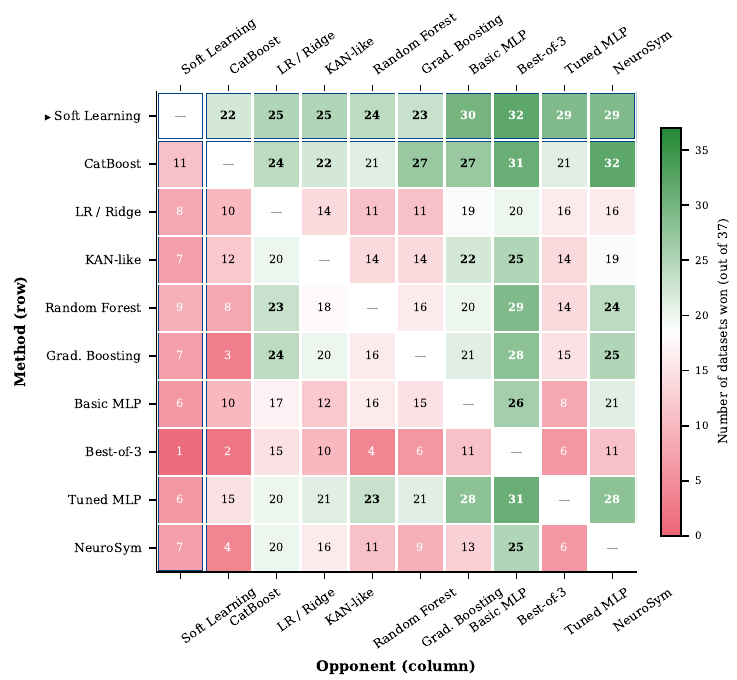}
\caption{\textbf{Supplementary Figure~1: Pairwise win--tie--loss matrix across all ten methods.} Each cell shows the number of datasets (out of 37) where the row method outperforms the column method (performance difference $> 0.001$). Green cells indicate a winning record ($> 18.5$ wins); red cells indicate a losing record. Diagonal cells (self-comparison) are marked ``---''. Blue border highlights Soft Learning, which achieves a winning record against every other method. The strongest dominance is against Best-of-3 (29 wins out of 37), Basic MLP (32 wins), and NeuroSym (29 wins). The closest competition is with CatBoost, where Soft Learning wins 22 of 37.}
\label{fig:supp_pairwise}
\end{figure}

\begin{figure}[H]
\centering
\includegraphics[width=\textwidth]{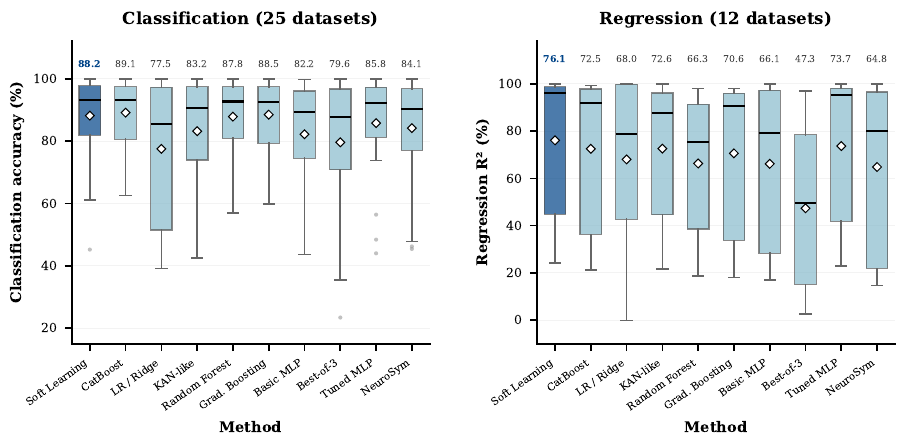}
\caption{\textbf{Supplementary Figure~2: Performance distribution across datasets for each method.} \textbf{a,}~Classification accuracy distribution across 25~datasets. \textbf{b,}~Regression $R^2$ distribution across 12~datasets. Box plots show median (horizontal line), interquartile range (box), 1.5$\times$IQR whiskers, and outliers (dots). White diamonds indicate mean values. Soft Learning (dark blue, leftmost) shows a compact distribution with high median on both task types, indicating consistent performance. Methods with wide boxes or long whiskers (e.g., LR/Ridge, Best-of-3) exhibit high variance---they perform well on some datasets but poorly on others. Methods sorted by number of first-place finishes.}
\label{fig:supp_boxplots}
\end{figure}

\begin{figure}[H]
\centering
\includegraphics[width=0.85\textwidth]{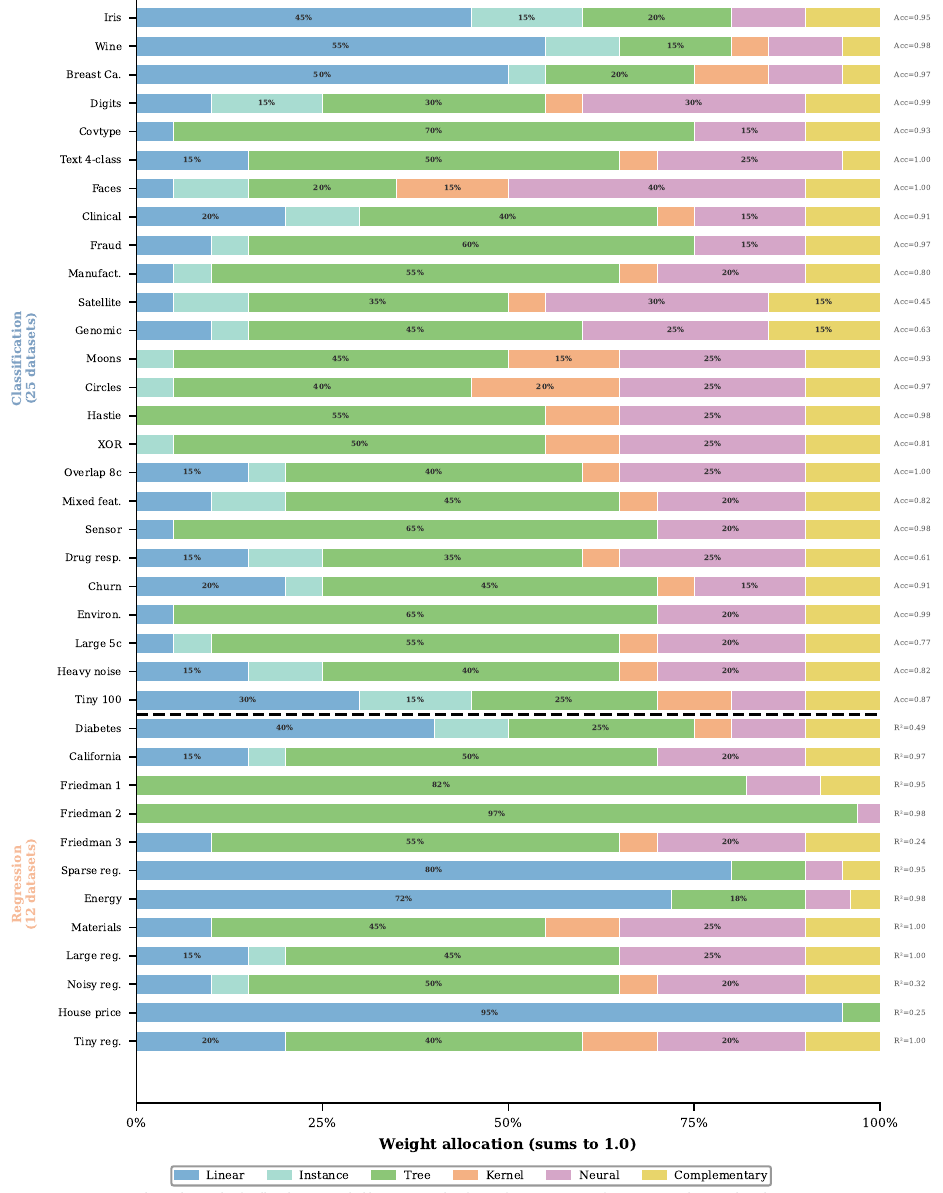}
\caption{\textbf{Supplementary Figure~3: Oracle-optimal specialist weights for all 37~datasets.} Each horizontal stacked bar shows the NNLS-optimized weight allocation across six specialist families (Linear, Instance, Tree, Kernel, Neural, Complementary) for one dataset. Dashed line separates classification (top, 25~datasets) from regression (bottom, 12~datasets). Performance metric (accuracy or $R^2$) shown on the right. Several patterns emerge: (i)~Tree specialists dominate most datasets, consistent with the documented superiority of tree-based methods on tabular data; (ii)~linear targets (House Price, Energy, Sparse Reg.) receive concentrated Linear weight; (iii)~complex problems (Satellite, Materials, Large Reg.) spread weight across 4--5~families; (iv)~Neural specialists receive substantial weight primarily on high-dimensional or pattern-rich datasets (Faces, Digits). This figure extends main-text Fig.~5, which shows six representative datasets.}
\label{fig:supp_weights}
\end{figure}

\end{document}